\documentclass[11pt]{article} 
\usepackage[margin=1in]{geometry}
\usepackage{amsfonts}
\usepackage{siunitx}
\usepackage{svg}
\usepackage{mathtools}
\usepackage{amsthm}
\usepackage{amsmath}

\usepackage{verbatim}
\usepackage{footnote}
\usepackage{subcaption} 
\usepackage{xcolor}
\definecolor{myblue}{rgb}{0.0, 0.2, 0.6}
\usepackage{dsfont}
\usepackage{amssymb}

\newcommand{\opt}{\mathrm{OPT}}
\newcommand{\card}{\DeclareMathOperator{\card}{\mathrm{Card}}}

\newcommand{\VC}{\mathrm{VC}}
\newcommand{\SP}{\mathrm{SP}}

\newcommand{\nonl}
{\renewcommand{\nl}{\let\nl\oldnl}}

\DeclareMathOperator*{\argmin}{arg\,min}
\usepackage{adjustbox}
\usepackage{caption}
\usepackage{array}
\usepackage{threeparttable}

\usepackage[utf8]{inputenc}
\usepackage{booktabs}

\usepackage[authoryear,sort]{natbib}
\setcitestyle{aysep={}}

\usepackage[
  colorlinks=true,
  citecolor=myblue,
  linkcolor=myblue,
  urlcolor=myblue
]{hyperref}
\usepackage[noframe]{showframe}
\usepackage{framed}
 {\endMakeFramed}
\definecolor{shadecolor}{gray}{0.95}
\usepackage{algorithm}
\usepackage{algpseudocode}

\usepackage{thmtools,thm-restate}
\newtheorem{theorem}{Theorem}
\newtheorem{lemma}[theorem]{Lemma}
  \newtheorem{example}{Example}
    \newtheorem{definition}{Definition}
  \newtheorem{proposition}[theorem]{Proposition}
  
  \newtheorem{corollary}[theorem]{Corollary}

  \newtheorem{conjecture}[theorem]{Conjecture}
  
\newtheorem{claim}[theorem]{Claim}

\renewenvironment{proof}[1][]{\par\noindent{\bf Proof #1\ }}{\hfill$\blacksquare$\\[2mm]}

\usepackage[authoryear,sort]{natbib}
\newcommand{\cA}{\mathcal{A}}

\newcommand{\cC}{\mathcal{C}}
\newcommand{\cD}{\mathcal{D}}
\newcommand{\cE}{\mathcal{E}}
\newcommand{\cF}{\mathcal{F}}

\newcommand{\cH}{\mathcal{H}}
\newcommand{\cI}{\mathcal{I}}
\newcommand{\cJ}{\mathcal{J}}
\newcommand{\cK}{\mathcal{K}}
\newcommand{\cL}{\mathcal{L}}

\newcommand{\cO}{\mathcal{O}}
\newcommand{\cP}{\mathcal{P}}

\newcommand{\cR}{\mathcal{R}}

\newcommand{\cT}{\mathcal{T}}
\newcommand{\cU}{\mathcal{U}}
\newcommand{\cV}{\mathcal{V}}

\newcommand{\cX}{\mathcal{X}}
\newcommand{\cY}{\mathcal{Y}}
\newcommand{\cZ}{\mathcal{Z}}

\newcommand{\A}{\mathbb{A}}

\newcommand{\N}{\mathbb{N}}

\newcommand{\R}{\mathbb{R}}

\makeatletter

\newcommand{\E}{\mathbb{E}}

\newcommand \prob {\mathop{{{\mathbb{P}}}}\nolimits}

\title{Auditing Fairness under Model Updates:\\ Fundamental Complexity and Property-Preserving Updates}

\author{
  \textbf{Ayoub Ajarra, Debabrota Basu}\\
  \'Equipe Scool, Univ. Lille, Inria, CNRS, Centrale Lille, UMR 9189- CRIStAL, France\\
  \small{\texttt{ayoub.ajarra@inria.fr}, \texttt{debabrota.basu@inria.fr}}
}

\date{}

\begin{document}

\maketitle

\begin{abstract}

As machine learning models become increasingly embedded in societal infrastructure, auditing them for bias is of growing importance. However, in real-world deployments, auditing is complicated by the fact that model owners may adaptively update their models in response to changing environments, such as financial markets. These updates can alter the underlying model class while preserving certain properties of interest, raising fundamental questions about what can be reliably audited under such shifts.

In this work, we study group fairness auditing under arbitrary updates. We consider general shifts that modify the pre-audit model class while maintaining invariance of the audited property. Our goals are twofold: (i) \textit{to characterize the information complexity of allowable updates, by identifying which strategic changes preserve the property under audit}; and (ii) \textit{to efficiently estimate auditing properties, such as group fairness, using a minimal number of labeled samples}.

We propose a generic framework for PAC auditing based on an \textit{Empirical Property Optimization (EPO)} oracle. For statistical parity, we establish distribution-free auditing bounds characterized by the SP dimension, a novel combinatorial measure that captures the complexity of admissible strategic updates. Finally, we demonstrate that our framework naturally extends to other auditing objectives, including prediction error and robust risk.
\end{abstract}
\tableofcontents
\newpage
\section{Introduction}
The rise of algorithmic decision makers in high-stakes domains such as healthcare, finance, and employment \citep{obermeyer2019dissecting, raji2019actionable, aboy2024navigating, montag2024successful} has made algorithmic regulation a central challenge. While regulation can take different forms~\citep{vecchione2021algorithmic}, accurately auditing properties of stochastic ML models has emerged as a fundamental component.  In trustworthy ML, the key focus is to ensure that ML models meet social and ethical constraints, such as fairness measures, robustness, and privacy guarantees, through auditing~\citep{raji2020closing, madaio2020co,le2023algorithmic}. This has spurred significant interest within ML communities in developing accurate and reliable auditing algorithms~\citep{kearns2018preventing,cohen2019certified,IFverification:jhon,BayesAudit:Neiswanger,yan2022active, hsu2024distribution,ajarra2024active, she2025fairsense}. However, real-world assessments of algorithmic bias become complex under model drift and can fundamentally alter the very properties under audit \citep{widmer1996learning, lu2018learning}. These shifts may occur in several scenarios. For instance, in evolving operational scales, a business (model owner) expands and gains access to vastly larger and more diverse datasets (e.g., customer transactions and market signals). Under such conditions, simple models like linear classifiers or decision trees may begin to underfit, failing to capture the emergent complexity of the data. Moreover, the trained model class itself may change. For example, a model trained under stable market conditions using linear assumptions may perform poorly when markets become highly volatile, as nonlinear dependencies tend to dominate \citep{cont2001empirical}. In such regimes, the model owner needs to update its model to flexible architectures such as neural networks often yield superior predictive performance \citep{goodfellow2016deep}. This trade-off also extends beyond statistical performance; complex models like deep neural networks typically demand significant computational resources, which may be disproportionate to the business need. A model update can occur, where a preferable strategy is to find an interpretable linear model \citep{rudin2019stop}. 

Given these interdependencies among data scale, dynamic environment, computational constraints, and interpretability requirements, performance audits are insufficient. In practice, comprehensive audits in regulated or centralized settings can take several months due to their labor-intensive nature and the need for cross-functional coordination across engineering, legal, and compliance teams \citep{raji2020closing}. This motivates the need for auditing frameworks that allow model updates without compromising the audited property. \cite{yan2022active} studies the auditing problem for Statistical Parity (Definition~\ref{def:SP}) under the assumption that the hypothesis class is known to the auditor. In their framework, strategic updates are formalized as manipulation-proofness: an updated model is considered manipulation-proof (does not alter the property under audit) if it belongs to the version space specified by a teaching set. Crucially, the size of this teaching set scales with the sample complexity required to reconstruct the model. Consequently, if the model owner is constrained from altering predictions on these audit points, the model itself remains essentially unchanged with high probability. This effectively precludes meaningful model updates during the audit process, rendering the manipulation-proofness condition overly restrictive in general applications. Because their approach focuses on reconstructing the model using the known hypothesis class, \cite{yan2022active} characterizes audit complexity under manipulation-proofness via classical learning-theoretic measures (specifically, the disagreement coefficient and Vapnik-Chervonenkis dimension). They leave open the question of whether alternative information-theoretic measures could yield a more flexible characterization of audit complexity. In this work, we focus on characterizing the information complexity of the strategic class. We ask:
\begin{center}\label{Q1}
\textbf{Q1:} \textit{Given a strategic class $\cF$, what is the information complexity of auditing group fairness for models in $\cF$ without full model reconstruction\footnote{e.g., in terms of learning-theoretic complexity measures such as VC dimension}?}
\end{center}

\textbf{Q1} seeks a combinatorial characterization of the classes of strategic updates $\cF$ for which group fairness auditing is feasible. In particular, it asks how the complexity of $\cF$ governs the amount of information required to certify fairness properties without learning the model itself.

As model classes grow increasingly expressive --- most notably with the use of neural networks that go far beyond linear decision boundaries --- a complementary question naturally arises:

\begin{center}
\textbf{Q2:} \textit{If the strategic class is allowed to grow arbitrarily in VC dimension, can strategic updates still admit information-theoretically feasible fairness auditing?}
\end{center}

In highly overparameterized regimes, learning is known to become information-theoretically hard, with VC dimension serving as a proxy for combinatorial complexity. \textbf{Q2} investigates whether fairness auditing may nevertheless remain strictly easier than learning, even for highly expressive model classes.

\subsection{Related Work: Black-box Auditing}

There are two main lines of work in this context. One that tries to \textit{verify} whether a property of an ML model shoots over a certain threshold. This has been extensively studied in the case of robustness~\citep{cohen2019certified,salman2019provably}, group fairness measures, like SP~\citep{albarghouthi2017fairsquare,BayesAudit:Neiswanger}, and individual fairness~\citep{IFverification:jhon}. \cite{kearns2018preventing,hsu2024distribution} aimed to verify (SP) through a reduction to weak agnostic learning. \cite{hsu2024distribution} studied this framework for Gaussian feature distributions and homogeneous halfspace subgroups, and demonstrate the problem’s computational hardness. Though the initial literature focused on verification, it requires a priori knowledge of a valid threshold, which is hard to pre-define without social and application context. The other line of works try to accurately and statistically \textit{estimate} the property under audit. \cite{BayesAudit:Neiswanger} proposes a Bayesian approach to estimate distributional properties. \cite{wang2022beyond} estimates simpler distributional properties, e.g. the mean, the median. For ML models, \cite{yan2022active} further propose active learning algorithms to sample efficiently and audit SP of a model under the assumption that its class is known a priori. Due to space constraints, we provide a detailed related work section in Appendix~\ref{app:RW}.

In contrast to existing work \cite{yan2022active, godinot2024under}, we make no assumption that the auditor has access to the model owner’s original hypothesis class (Figure \ref{fig:limit}). The auditor is instead provided with a new strategic class representing the model owner’s intended post-audit constraints. The auditing task then consists of identifying a model within this class that reproduces the discriminative behavior of the black-box model. This formulation explicitly models structural shifts induced by the model owner during the audit, positioning the auditor as both an evaluator of group fairness and a mechanism for identifying feasible alternative models in a dynamic setting.

\begin{figure*}[t]
    \centering
    \includegraphics[width= 0.7\textwidth]{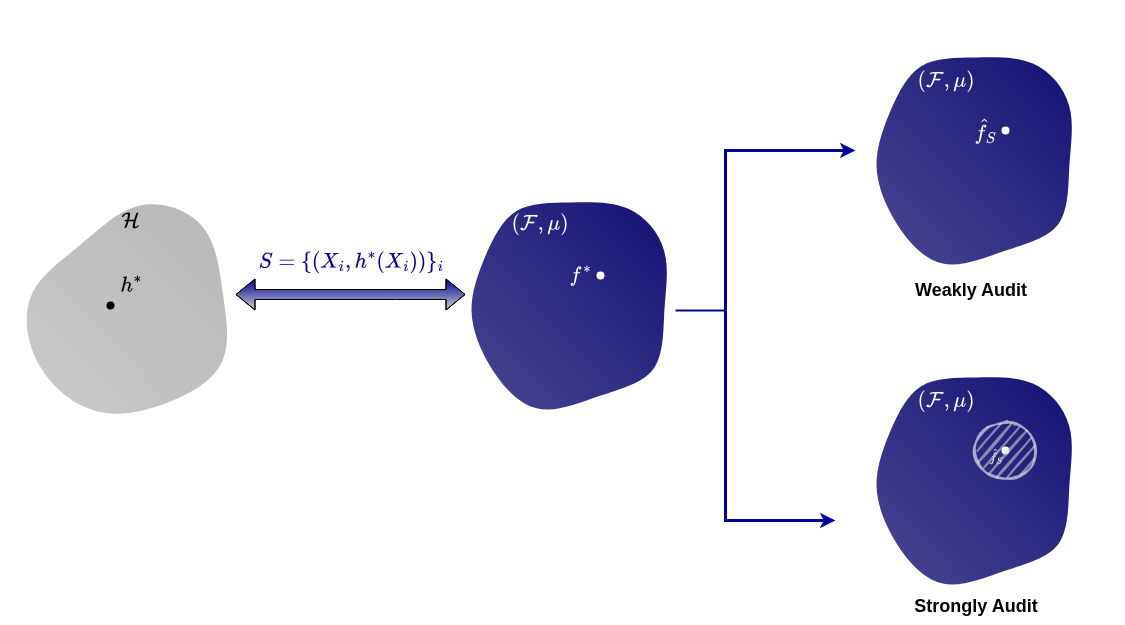}
    \caption{Grey shapes represent objects unknown to the auditor, while blue shapes denote model classes known to the auditor. The model class under audit is unknown, whereas the strategic model class is given to the auditor. In weak auditing, the auditor seeks a model in the strategic class with the same group fairness value as the audited model. In strong auditing, the auditor aims to characterize the set of all models in the strategic class that share the same group fairness value.}\label{fig:limit}
\end{figure*}

\subsection{Contributions}

\textbf{A New Framework.}  Our work introduces a general framework for the auditing problem under strategic updates that applies to \textit{unknown} strategic classes. We shift the goal of the auditor from merely estimating the property value to also requiring the auditor to output a subset of the strategic class that preserves this property under audit (\textit{prospect class}, definition~\ref{def:prospects}). This generalizes the standard notion of manipulation-proof to settings involving broader forms of model updates.

1. \textbf{Universal auditor.} We propose a generic algorithm (Algorithm~\ref{algo}) framework that, given access to the property of interest, outputs an estimate of the property and the prospect class. Specifically, we define auditing losses for different properties: SP, learning error, learning stability, and robust risk. 

2. \textbf{Fairness audits and SP dimension.} We focus on statistical parity as our primary property of interest. We provide sample complexity guarantees for finite and infinite hypothesis classes. We introduce a new capacity measure called the SP dimension for weakly auditing infinite hypothesis classes (Definition \ref{def:auditype1mp}). We demonstrate its relationship to VC dimension, showing that VC dimension effectively captures a broader set of auditing tasks beyond traditional statistical learning frameworks. For multiple protected groups, we prove that any shattered set must include instances from distinct protected groups, providing new insights into multi-group auditing scenarios.

3. \textbf{Measuring coverage of prospect class.} We introduce the $\mu$-prospective ratio as a data-dependent measure to measure the coverage of the prospect class. For statistical parity, we provide concentration bounds on the estimation of the prospect ratio.

Finally, we numerically demonstrate that our framework yields good estimates of SP and prospect class for real-life datasets and multiple ML models. 

\section{Preliminaries}
In this section, we present the problem formulation and introduce illustrative examples of properties of ML models, leading to a formalization of the PAC auditing framework under model updates in two distinct settings (i.e., weak and strong auditing).

Let $\cX$ and $\cY$ be the input and output spaces of an ML model, respectively. The input data follows a distribution $\cD_{\cX}$. $\cD$ denotes the joint distribution over the product space $\cX \times \cY$ modeling the randomness of the black-box model under audit \footnote{Rather than assuming a marginal distribution over inputs with deterministic labels, we model the data using a joint distribution, which allows us to capture scenarios involving randomized predictors.}, and $\cF$ denotes the strategic class of models from $\cX$ to $\cY$. A property of the ML model is defined as a functional $\mu: \cF \times \cP \rightarrow \R$ that takes the model and the corresponding data-generating distribution to yield a real number, where $\cP$ is a class of generating distributions.

\paragraph{Different Distributional Properties of ML Models.}  Many such properties discussed previously have been studied in the literature. In the following paragraph, we highlight several examples, focusing in particular on those defined over multiple subpopulations via conditional distributions. Statistical parity is one such property, measuring disparities in positive prediction rates between protected groups.~\citep{feldman2015certifying}. 

\begin{definition}[Statistical Parity]\label{def:SP}
    Let $\{\cX_0, \cX_1\}$ denote a partition of the input space based on a binary protected attribute (e.g., gender, where $\cX_0$ corresponds to females and $\cX_1$ to males). 
    The statistical parity of a model $f \in \cF$ measures the discrepancy of its predictions with respect to the two protected groups $\cX_0$ and $\cX_1$, i.e. $\mu(f,\cD) \triangleq |\underset{x \sim \cD_{\cX}}{\prob} [ f(x) = 1 | x \in \cX_0   ] -  \underset{x \sim \cD_{|\cX}}{\prob}[ f(x) = 1 | x \in \cX_1]|$.
\end{definition}

A model is considered fair across protected groups if it achieves a small statistical parity value. Definition \ref{def:SP} can be extended to multiple protected groups by considering the maximum discrepancy between groups.

\textbf{Expected risk.}  Measures the expectation of prediction errors by a model $f$ on input-output pairs generated by $\cD$~\citep{mlbook}. For the class of binary classifiers,
$\mu(f,\cD) = \underset{(x,y) \sim \cD}{\prob}[f(x) \neq y]$. 

\textbf{Learning stability.} Measures the discrepancy in predictive performance across two environments.  For a model trained on a distribution $\cD_{\text{src}}$ and deployed on $\cD_{shift}$. We define its learning stability as $\mu(f, \cD_{\text{src}}, \cD_{\text{shift}}) \triangleq \Big| \underset{(x,y) \sim \cD_{\text{src}}}{\prob}[f(x) \neq y ] - \underset{(x,y) \sim \cD_{\text{shift}}}{\prob}[f(x) \neq y ] \Big|$. This definition is closely related to distribution shifts; however, distribution shift itself is a property of the data alone:  $\cD_{\text{src}} \neq \cD_{\text{shift}}$ \citep{ben2010theory,taori2020measuring}. Our definition evaluates how stable the model’s performance remains under such shifts. Details are given in Appendix~\ref{app:examples}.

\textbf{Robust Risk.} After deploying an ML model, we encounter inputs that are noisy or manipulated by an adversary. To ensure safety, it is necessary to verify that the ML model is robust against such perturbations of the input. Robustness is measured using robust risk and is further ensured by minimizing the robust risk (e.g., it works on adversarial ML ~\citep{adversarialml}). For any $x$ in $\cX$, let $\cU(x)$ denote the set of perturbations acting on input $x$, and let $\cD^*$ denote the resulting distribution after the true model deployment.  Robust risk is defined as $\mu(f, \cD, \cU) = \underset{(x,y) \sim \cD}{\E} \Big[\underset{z \in \cU(x)}{\sup} \mathds{1}_{f(z) \neq y}   \Big]$.

\paragraph{Algorithmic Auditing.} 

As discussed in the introduction and illustrated in Figure \ref{fig:auditmot} and \ref{fig:limit}, the auditor is given access to a strategic class $\cF$, takes as input a set of i.i.d samples and the desired property $\mu$, and outputs both an estimate of the property and a property-preserving class that we refer to as the prospect class. Formally, an auditing problem is defined as a quintuplet $\langle \cX, \cY, \cF,\cP, \mu, \ell \rangle$, where $f$ is the model under audit. $\ell$ is a loss that depends \textit{implicitly} on the property under audit $\mu$. 

Details on how the loss function depends on the audited property are provided in Appendix~\ref{app:examples}, where we define loss functions for the properties discussed above and show how our framework extends to broader auditing problems.  In the following, we make no assumptions about $f$. In particular, we do not assume knowledge of the model class. Indeed, $f$  may be a randomized machine learning algorithm whose behavior depends on an unknown randomization mechanism. In the next section, we introduce a novel auditing setting that not only evaluates existing models but also yields prospective machine learning models for future deployment (aka Prospect class, see Definition \ref{def:prospects}). We establish a formal connection between auditing properties of black-box models with prospective guarantees and agnostic learning of those properties from data within a strategic class $\cF$. 

\begin{figure*}[t]
    \centering
    \includegraphics[width= 0.7\textwidth]{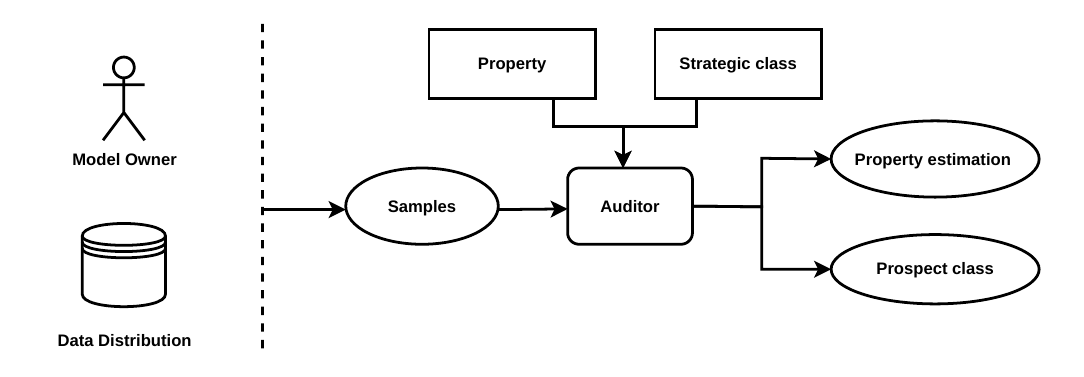}
    \caption{A schematic of black-box auditing with prospects.}\label{fig:auditmot}
\end{figure*}

\section{Auditing with Prospects}\label{sec:formulation}

During the auditing process, the model owner provides the auditor with a finite set of i.i.d samples labeled by a black-box model. Subsequently, the owner may wish to update their decision rule; this strategic update is communicated to the auditor. To accommodate such strategic updates, the auditor identifies the prospect class, a subclass of the strategic class that preserves the audited property. Agnostic auditing refers to the scenario where the pre-audit model is not an element of the strategic class.
The prospect class must include models that have the same property as the pre-audit model (defined as $\opt(\mu,\cD, \cF) = \min_{f \in \cF} |\mu_{\cD}(f) - \mu(\cD)|$). Realizable auditing corresponds to the scenario where the pre-audit model belongs to the strategic class, and in this case $\opt(\mu,\cD, \cF) =0$. This motivates Definition~\ref{def:auditype1mp}, where in weak auditability, finding one model in the prospect class suffices. For a fixed property $\mu$, i.i.d sample $S$ and a class $\cF$, let $\opt(\mu,S, \cF) \triangleq \min_{f \in \cF} |\mu_{S_{\cX}}(f) - \mu(S)|$. Finally, for an auditing algorithm $\cA$ taking as input an i.i.d set $S$ of size $m$, we define the audit risk for $\mu$ as: 
\begin{align}\label{eq:audit_risk}
\cE_m(A[S], \mu) \triangleq | \mu_{\cD}(\cA[S]) - \opt(\mu,\cD, \cF) |\,.\tag{Audit Risk}
\end{align}

\begin{definition}[Weakly $\mu$-auditable class]\label{def:auditype1mp}
    We call an algorithm $\cA$ to be  $(\epsilon, \delta)$-weak auditor for a given auditing problem $\langle \cX, \cY, h, \cF,\cP, \mu\rangle$ when it uses a sample set $S$ of size $m$ sampled from any $\cD \in \cP$ to yield a model $\cA[S]$ and an estimate $\mu_{\cD}(\cA[S])$ satisfying:
    \begin{align*}
        \underset{S \sim \cD^m}{\prob}  \Bigl [  \cE_m(A[S], \mu) \geq \epsilon \Bigl] \leq \delta\,,
    \end{align*}
    where $\epsilon, \delta \in (0,1)$ and $m$ is a bounded function of the problem parameters and $(1/\epsilon, 1/\delta)$.     
    We say that a strategic class $\cF$ is weakly $\mu$-auditable if for all $(\epsilon, \delta) \in (0,1)^2$, there exists an $(\epsilon, \delta)$-weak auditor.
\end{definition}

\setlength{\textfloatsep}{6pt}
\begin{algorithm}[H]
    \caption{EPO Oracle}\label{algo}
    \begin{algorithmic}[1]
        \Require 
        Training dataset $\mathcal{D} = \{(x_i, y_i)\}_{i=1}^{m}$, Strategic class $\cF$, Property to audit $\mu$
        \Ensure Estimate of the property $\hat{\mu}$, prospective model $\hat{f}$
        \State Define empirical risk for $\mu$: $\cE_m(f,\mu)$ 
        \State Use an ERM oracle to solve the optimization problem: $ \hat{f} = \arg\min_{f \in \cF} \cE_m(f,\mu) $
        \State \Return $\hat{f}$ and $\hat{\mu}_m(\hat{f})$
    \end{algorithmic}
\end{algorithm}

\paragraph{From ERM to Auditors: Auditing with loss functions.}Algorithm~\ref{algo} proposes a generic framework for using ERM oracles to weakly audit any property $\mu$ of an ML model $f$ given a strategic class $\cF$. Broadly, the intuition is that if we can define a loss function to audit each of the properties, we can use the samples collected from the black-box model $f$ to learn a prospective model $\hat{f} \in \cF$. One can use any off-the-shelf ERM solver, like SGD, Adam, ADMM etc., in Line 2.

The following definition extends weak auditability by saturating the strategic class and identifying all models within the prospect class that satisfy the audited property.

\begin{definition}[Strongly $\mu$-auditable class]\label{def:auditype2mp}
    We call an algorithm $\cA$ to be  $(\epsilon, \delta)$-strong auditor for a given auditing problem $\langle \cX, \cY, \cF,\cP, \mu\rangle$ when it uses a sample set $S$ of size $m$ sampled from any $\cD \in \cP$ to yield a class of models $\cA[S] \subseteq \cF$ and an estimate $\mu_{\cD}(\cA[S])$ satisfying:\\
(i) \textbf{Correctness:} $$\underset{S \sim \cD^m}{\prob}  \Bigl [ \sup_{f \in  A[S]} \cE_m(f,\mu)  \geq \epsilon \Bigl] \leq \delta\,.$$\\
(ii) \textbf{Completeness:} $$\underset{S \sim \cD^m}{\prob}  \Bigl [ \inf_{f \in {A^C[S]}}  \cE_m(f,\mu)  \leq \epsilon 
   \Bigl] \leq \delta\,.$$\\
    $A[S] \subseteq \cF$ is a subclass of models that achieve a small error on the training set $S$. 
    We say that a hypothesis class $\cF$ is strongly $\mu$-auditable if for all $(\epsilon, \delta) \in (0,1)^2$, there exists an $(\epsilon, \delta)$-strong auditor.
\end{definition}

Strong auditability is governed by two key conditions: correctness, which requires all models in the prospect class to preserve the same property as the pre-audit model, and completeness, which ensures the hypothesis class has been fully exhausted by including all models that share the property $\mu$ with the pre-audit model. The resulting prospect class serves as the set of potential post-audit models \footnote{The model owner may implement a majority vote among hypotheses in the prospect class as the new post-audit model; this can be referred to as improper auditing.}.

\section{Auditing Statistical Parity}\label{sec:mainres}
In this section, we focus on characterizing the problem of auditing statistical parity. In the black-box setting, the auditor can only access a pool of labeled instances by $f$, which we denote $S$ (i.e., $S \subseteq \cX \times \cY$). This is equivalent to assuming access to an empirical distribution $\hat{\cD}_{S}$ that encodes the discriminative behavior of the black box model with respect to the fixed protected groups. From now on, we denote the protected groups by $\cX_0$ and $\cX_1$, which form a partition of the input space $\cX$. Similarly, for a finite sample $S$, let $S_0 \subseteq \cX_0$ denote samples from the first protected group and $S_1 \subseteq \cX_1$ samples from the second protected group.

\subsection{Weakly Auditable Classes}
    We characterize the complexity using the minimum number of samples required for each protected group ($m_0$ and $m_1$). This approach differs from existing methods that rely on the total sample size ($m = m_0 + m_1$), which requires assuming specific proportions between protected groups \cite{yan2022active}. By considering the minimum samples needed for each group independently, our method remains valid regardless of how the probability mass is concentrated across protected groups, making it more robust to group imbalances in real-world scenarios. To rigorously define PAC weakly auditing of statistical parity, we denote by $\cF \subseteq 2^{\cX}$ the strategic class (a generalization of the hypothesis class). With this structure, we present the strategic lemma linking EPO solvers over pairs of points from different protected groups to uniform convergence in statistical parity.

\begin{restatable}[Strategic Lemma]{lemma}{strategiclemma}\label{SP-strategiclemma}
Let $\epsilon, \delta \in (0,1)$, $m: (0,1)^2 \to \N$. Suppose that the following holds:
\begin{itemize}
\item \textbf{Estimation accuracy.} $\cA$ outputs $f_S$ from $\cF$: 
    $\underset{S \sim \cD^m}{\prob} \Bigl[| \hat{\mu}_S (f_S) - \widehat{OPT}(S, \cF)| > \frac{\epsilon}{3} \Bigl] < \frac{\delta}{2}.$
\item \textbf{Uniform convergence.}  $\cF$ verifies 
    $\underset{S \sim \cD^m}{\prob} \Bigl[\exists f \in \cF, | \mu_{\cD} (f) - \hat{\mu}_S(f)| > \frac{\epsilon}{3} \Bigl] < \frac{\delta}{2}. $
\end{itemize}
Then, $\cA$ is $(\epsilon, \delta)$-weak auditor for statistical parity with a sample complexity.
\end{restatable}

Lemma~\ref{SP-strategiclemma} establishes that uniform convergence over the property, combined with empirical audit risk minimization, guarantees weak auditing PAC auditing of that property. This lemma is risk-free, unlike the one used to derive VC bounds, which uses the intermediate of a loss function that we do not consider here. The proof is given in Appendix~\ref{app:strategiclemma}.

\paragraph{Finite Strategic Class.}

We first present a result on the sample complexity for weakly auditable finite classes.
\begin{restatable}[Agnostic weak auditability]{theorem}{theoremweakfinite}\label{SPauditfiniteH}
If $\cF$ is a finite hypothesis class, than $\cF$ is weakly auditable, with respect to statistical parity, for any distribution on $\cX \times \cY$ with a sample complexity $\cO \Bigl( \Bigl\lceil \frac{18}{ \epsilon^2} \log \frac{8 |\cF|}{\delta} \Bigl \rceil  \Bigl)$.
\end{restatable}
Theorem~\ref{SPauditfiniteH} offers an intuitive understanding of the auditing hardness, highlighting in the case of a finite class that auditing SP using weak auditing framework requires more sample complexity than both learning and reconstructing the model. The proof is given in Appendix~\ref{prooflem1}.

\textbf{Infinite Hypothesis Class.} The finiteness of $\cF$ poses inherent limitations for auditing, as such a hypothesis class may not contain models that satisfy the desired post-audit stakes. This may result in a singleton prospect class (pre-audit model), rendering the auditing process ineffective. Therefore, we extend our analysis to infinite strategic classes. Here, VC dimension falls short in tightly measuring the complexity of SP auditing. We illustrate this in the following example:
\begin{example}\label{example:3pt2groups}
    Consider a set $\cX \subseteq \R^2$ where the protected attribute is 'gender' and the second feature is 'age'. In the context of classification in $\R^2$ using linear classifiers (Figure~\ref{fig:spexample}), it is known that the VC dimension of this class is three. However, if the three points that can be shattered by the class of linear classifiers share the same protected attribute (gender), they are collinear and cannot be shattered.
\end{example}
Next, we define group traces of the strategic class $\cF$ with respect to the protected groups $\cX_0$ and $\cX_1$.
\begin{definition}[Group-Traces of a strategic class]\label{def:sptraces}
    Let $\cX$ denote an uncountable space, $\cF$ a set of subsets of $\cX$, and $S$ a finite subset of $\cX$. The group-traces of $\cF$ in the protected groups of $S = S_0 \cup S_1$, denoted by $ \Delta_{\cF}^{SP}(S)$, is defined as $\Delta_{\cF}^{SP}(S_0, S_1) \triangleq \Big \{ (A_0,A_1)| A_0 \subseteq S_0, A_1 \subseteq S_1, \exists c \in \cF, A_0 = c \cap S_0, A_1 = c \cap S_1  \Big \}$.
\end{definition}

Intuitively, the set of group-traces of a concept class $\cF$ represents all possible discriminatory behaviors within $\cF$ with respect to the protected groups.

\begin{figure}[t!]
    \centering
    \includegraphics[width= 0.5\textwidth]{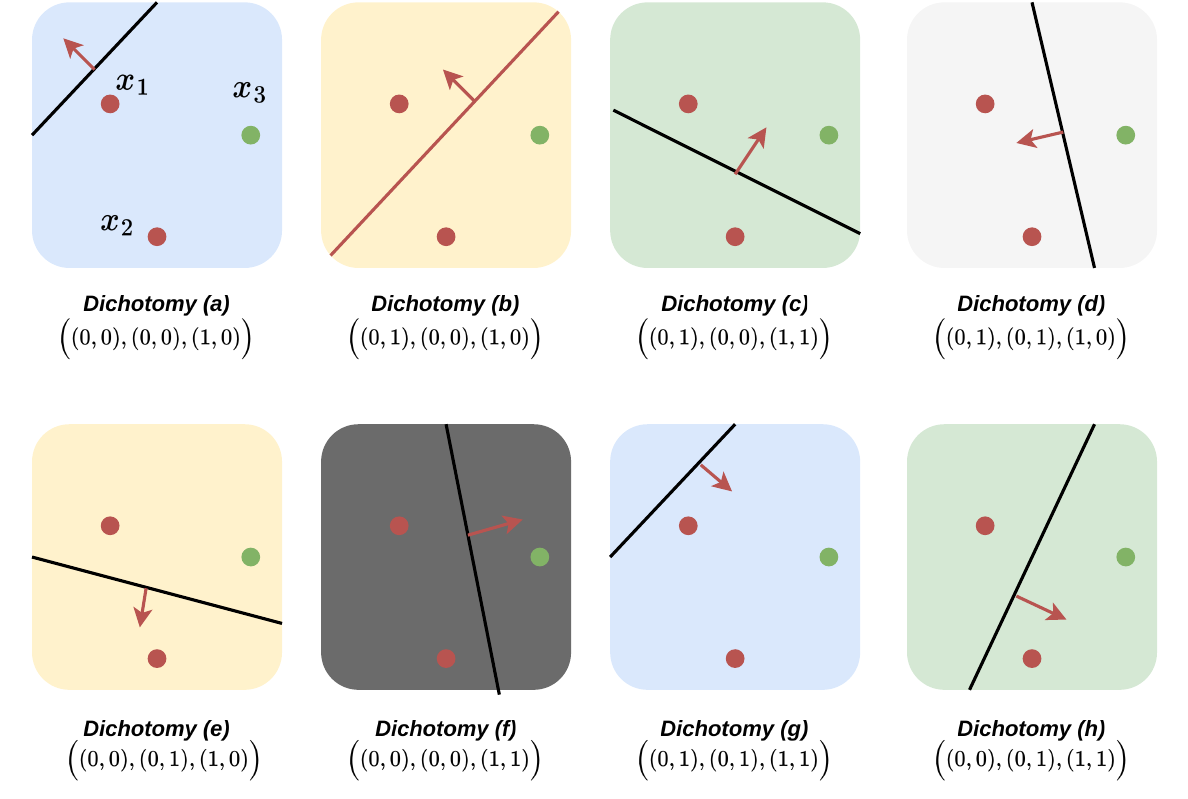}
    \caption{Illustration of SP dimension for the case of non-homogeneous classifiers in $\R^2$}\label{fig:spexample}
\end{figure}

\begin{lemma}\label{lemma:welldef}
    For any finite and non-empty sample sets $S_0$ and $S_1$ drawn from the first and second protected groups, respectively, the set $\Delta_{\cF}^{SP}(S_0, S_1)$ is well-defined.
\end{lemma}

\begin{proof}
     To establish the result in Lemma~\ref{lemma:welldef}, it is sufficient to show that, for any concept class $\cF$ of $\text{VC}(\cF) = d$, and for any shattered set $S = \{x_1, \cdots, x_d\}$ shattered by $\cF$, not all elements of $S$ can belong to the same protected group. By the definition of VC dimension, the total number of dichotomies $u$ for the sample $S$ is $2^d$.  We assume by contradiction that all elements of $S$ belong to a single protected group. Without loss of generality, we assume that $S \subseteq \cX_0$. Let $x'$ be any point from $\cX_1$.
     For all $i$ in $[2^d]$, let $c_i$ denote the concept from $\cF$ that realizes the dichotomy $u_i$. Since $\cX_0$ and $\cX_1$ form the set of components of $\cX$ ($\cX_0 \cap \cX_1 = \emptyset$), each $c_i$ can be extended to $\Tilde{c^0_i}$ and $\Tilde{c^1_i}$, such that $x' \in \Tilde{c^0_i}$ and $x' \notin \Tilde{c^1_i}$. 

    Hence, $\cF$ realizes $2^{d+1}$ dichotomies over $S \cup \{x'\}$. In other words, $S \cup \{x'\}$ shatters $\cF$. This is a contradiction because $\text{VC}(\cF) =d$.
 \end{proof}

\begin{definition}[SP dimension]\,\label{def:spdim}
We say that a sample $S$ SP-shatters by $\cF$ when $|\Delta^{\SP}_{\cF}(S_0,S_1)| = 2^{|S|} + |S| - 2^{|S_0|} - 2^{|S_1|}$. 
The SP dimension of a class $\cF$ of subsets of $\cX$ is 
    \begin{equation*}
        \SP(\cF) \triangleq \underset{|S|: S = S_0\cup S_1}{\max} \log_2 |\Delta^{\SP}_{\cF}(S_0,S_1)| \,.
    \end{equation*}
\end{definition}

In contrast to \cite{yan2022active}, whose approach relies on pre-audit model reconstruction via a teaching set and is characterized by classical learning-theoretic complexities (e.g., the disagreement coefficient and VC dimension), the SP dimension of a concept class $\cF$  captures only those group-wise dichotomies that exhibit distinct discriminatory behaviors. It achieves this by quotienting out redundant symmetries between protected groups that induce equivalent discrimination patterns within $\cF$.  

\begin{example}
    As illustrated in Figure \ref{fig:spexample}, for non-homogeneous linear classifiers in $\R^2$, the VC dimension is $3$, which implies $2^3 = 8$ possible dichotomies. Given Lemma \ref{lemma:welldef} requiring at least one point from each protected group to be in the shattered set, the figure depicts this configuration. The same-colored squares in Figure~\ref{fig:spexample} illustrate the same behavior of dichotomies; Consider the green square as an example: it demonstrates a specific behavior with respect to protected groups by assigning a positive label to only one point from the first protected group. This constraint reduces the total number of valid dichotomies from 8 to 5. This is equivalent to computing SP dimension using definition \ref{def:spdim}: $2^{\SP(\cF_2)} = 2^3 + 3 - 2^2 - 2^1 = 5$.
\end{example}

\begin{theorem}[Quantitative characterization]\label{theorem:spbounds}
    For any concept class $\cF$, the minimum number of samples $ m(\cF, \epsilon, \delta)>0$ required to weakly audit SP is lower bounded by $\Omega\left(\frac{\SP(\cF)}{\epsilon^2}\right)$, while Algorithm~\ref{algo} requires $\cO \Big( \frac{32}{\alpha (1 - \alpha) \epsilon^2} \max\{ \log \frac{2}{\delta}, 2 \SP(\cF) \log \frac{32e}{\epsilon^2}  \} \Big)$ samples. Here, $\alpha \in(0,1)$ is the ratio of samples from two protected groups. 
\end{theorem}
Theorem~\ref{theorem:spbounds} characterizes the sample complexity bounds for weak auditability when the auditor is given $m$ samples, distributed such that $\alpha m$ samples come from the first group and $(1-\alpha)m$ samples from the second group. For this setting, the theorem establishes both necessary and sufficient conditions: a sample size of $\Omega(\frac{\SP(\cF)}{\epsilon^2})$ is necessary for weak auditability, while a sample size of $\cO \Big( \frac{32}{\epsilon^2} \max\{ \log \frac{2}{\delta}, 2 \SP(\cF) \log \frac{32e}{\epsilon^2}\}\Big)$ is sufficient.

\begin{corollary}[Qualitative characterization]
    A concept class $\cF$ is agnostic and realizably weak auditable if and only if $\SP(\cF)$ is finite.
\end{corollary}

The proofs are given in Appendix~\ref{sec:agnostic_sp_mp_type1}. These bounds establish that weak auditability has a complexity measure similar to learnability. Specifically, the learning problem of a hypothesis class $\cF$ has a sample complexity upper bound of $O\left(\frac{\VC(\cF) + \log(1/\delta)}{\epsilon^2}\right)$ and a lower bound of $\Omega\left(\frac{\VC(\cF)}{\epsilon^2}\right)$. These bounds show a tight correspondence with the bounds for weakly auditing. Although, as discussed previously, the SP dimension is upper bounded by the VC dimension. Meaning that the SP dimension allows for fewer possible group behavioural dichotomies compared to classification dichotomies. The sample complexity, however, remains the same up to constant factors.

\begin{restatable}[Learnability vs. Auditability]{proposition}{rsvcsp}\label{prop:vcsp}
For any $\cF$ with finite VC and SP-dimensions,  $\VC(\cF) \geq \SP(\cF)$.
\end{restatable}

To comprehend this inequality, one can begin with a shattered set and reduce the size of dichotomies leading to the same discriminative behavior. The set of dichotomies in the concept class $\cC$, when the property to audit is the $\cL^1_{\cD}$ error, are defined as $\Delta_{\cC}(S) = \{A: A \subseteq S, \exists c \in \cC, A =c \cap S\}$. Then, the inequality is a direct consequence of the definitions of VC and SP dimension derived for the aforementioned dichotomies. 

Proposition~\ref{prop:vcsp} shows that auditing statistical parity has a lower information complexity than learnability. This difference arises because the auditing problem takes advantage of the input space's structure, specifically the symmetry between protected groups. This symmetry makes it easier to analyze how models behave with respect to protected groups, compared to the more complex task of learning the model's behavior across the entire input space.

\subsection{Strongly Auditable Classes}\label{sec:stronglyrslts}
We begin by analysing the bounds for strongly auditing the finite hypothesis classes.

\begin{restatable}[Strongly auditable finite classes]{theorem}{stronglyfinite}\label{theorem:stronglyfinite}
    Any finite hypothesis class is strongly auditable with a sample complexity $\cO \Bigl(\max \Big\{\underbrace{ \textstyle \frac{1}{\epsilon^2} \log \frac{|\cF|}{\delta}}_\text{without prospect}, \underbrace{\textstyle \frac{1}{\log\frac{1}{\epsilon^2}} \log \frac{|\cF|}{\delta}}_{\text{with prospect}}   \Big\}\Bigl)$.
\end{restatable}

The provided bounds highlight the inherent difficulty of covering the entire prospect class in terms of error when the hypothesis class is finite. This shows the fundamental trade-off of achieving strong auditability (correctness: $\frac{1}{\epsilon^2}$ and completeness: $\frac{2}{\log \frac{2}{\epsilon^2
}}$) that imposes an additional cost on sample complexity. The proof is in Appendix~\ref{app:stronglyfinite}.

The completeness constraint on strong auditability for infinite model classes presents a fundamental challenge (details in Appendix \ref{app:extexp}). To address it, we introduce a new characterization based on the prospect ratio. We begin by formally defining the prospect class, which serves as building blocks for the prospect ratio definition that follows.

\begin{definition}[Prospect class]\label{def:prospects}
    Given a distributional property $\mu,$ strategic class $\cF$, and a parameter $\epsilon \in (0,1)$, 
    the true prospect class with respect to $\mu$ is the subclass of models in $\cF$ that have the same property $\mu$ up to $\epsilon$: $
            \cP(\cF, \epsilon) \triangleq \{f \in \cF: |\mu(f) - \mu(f^*)| \leq \epsilon\}$.\\
    The empirical prospect class with respect to $\mu$ is the subclass of models in $\cF$ that have the same property $\mu$ on $S$, up to $\epsilon$: $
            \hat{\cP}(\cF, \epsilon) \triangleq \{f \in \cF: |\hat{\mu}(f) - \mu(f^*)| \leq \epsilon\}\,,$
    where $f^* \in \underset{f \in \cF}{\argmin} |\mu(f) - \mu^*|$, $\mu^*$ is true property of the black-box model, and $\epsilon$ is the prospect parameter.
\end{definition}

\begin{definition}[Prospect ratio]\label{def:propratio}
    The prospect ratio is the volume of the prospect class compared to the hypothesis class $\cF$, i.e. $r(\epsilon) \triangleq \frac{\cV(\hat{\cP}(\cF, \epsilon))}{\cV(\cF)}$. $\cV: \cF \to \cR^{+}$ is volume function.
\end{definition}

To address the challenge of infinite prospect classes, we need a systematic way to measure their volume. This can be achieved by reducing the search space to a finite set of representative models through a probability measure over the model class $\cF$. Let $\nu$ denote this probability measure on $\cF$. While the choice of $\nu$ affects the prospect ratio's value and is largely arbitrary, it enables us to work with a manageable subset of models. When $\cF$ is finite, a natural choice of $\nu$ is the uniform distribution, which reduces to the case of strongly auditing finite classes. Let $n$ denote the sample size of models drawn from the strategic class. The prospect ratio is influenced by two distinct sources of uncertainty: (1) uncertainty due to sampling models from the prospect class and (2) uncertainty due to sampling points from protected groups. This induces two types of errors. Let $f_1, \cdots, f_n$ be $n$ models sampled independently from $\nu$,  and $\Tilde{r}_n(\epsilon) \triangleq \frac{1}{n} \sum_{i=1}^n \mathds{1}_{f_i \in \cP(\cF, \epsilon)}, \hat{{r}}_n(\epsilon) \triangleq \frac{1}{n} \sum_{i=1}^n \mathds{1}_{f_i \in \hat{\cP}(\cF, \epsilon)}$. 
We define \textit{the estimator for the prospect ratio} as:

$$\hat{{r}}_{n,m_0,m_1}(\epsilon) \triangleq \frac{1}{n} \sum_{k=1}^n \mathds{1}_{|\frac{1}{m_0} \sum_{i=1}^{m_0} \mathds{1}_{f_k(x_i) = 1} - \frac{1}{m_1} \sum_{j=1}^{m_1} \mathds{1}_{f_k(x'_j)=1}| \leq \epsilon},$$
where $x_i$'s and $x'_j$'s are samples from the first and second protected group, respectively.

\begin{restatable}[Concentration of prospect ratio]{theorem}{asymsize}\label{theorem:prospectratio}
    For all $\epsilon, \upsilon, \tau \in (0,1)$, 
    $\prob \{r(\epsilon - \upsilon) - \tau \leq \hat{{r}}_{n,m_0,m_1}(\epsilon) \leq r(\epsilon + \upsilon) + \tau \} \geq \Big(1 - \exp \Big\{\frac{- 2 \upsilon^2 m_0 m_1}{n(m_0 + m_1)}  \Big\} \Big)^n ( 1 - \exp (-n\tau^2))^2$.
\end{restatable}
The prospect ratio estimation error exhibits exponential dependence on the size of the prospect class. While obtaining sufficient samples from each protected group is important for accurate statistical parity estimation, the sample size requirement with respect to models from the hypothesis class is exponentially more critical for achieving strong auditability.
The proof is given in Appendix~\ref{app:prospectratio}.

\paragraph{Infinite VC Classes.} The following proposition \ref{prop:infinitevcnotsp} provides a lower bound result to \textbf{Q2} posed in the introduction. This question is equivalent to asking whether we can reduce an infinite set of dichotomies to a finite set that captures all distinct behaviors with respect to protected groups. 

\begin{restatable}{proposition}{infinitevcnotsp}\label{prop:infinitevcnotsp}
    The model classes with infinite VC dimension are not weakly or strongly auditable.
\end{restatable}

The proof is given in Appendix~\ref{app:infinitevcnotsp}. In practice, models are parameterized and thus have a finite VC dimension. However, from a theoretical standpoint, neural networks with infinite width or infinitely many parameters fall outside the practical scope of auditing statistical parity.
\section{Numerical Experiments}\label{sec:experiments}
\begin{figure*}[t!]
\centering
\includegraphics[width=0.22\textwidth]{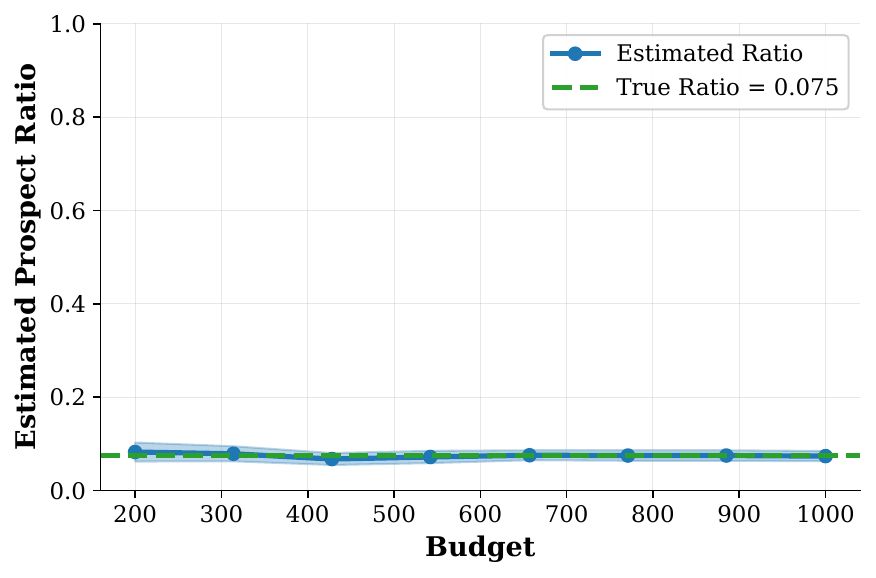}\hfill
\includegraphics[width=0.22\textwidth]{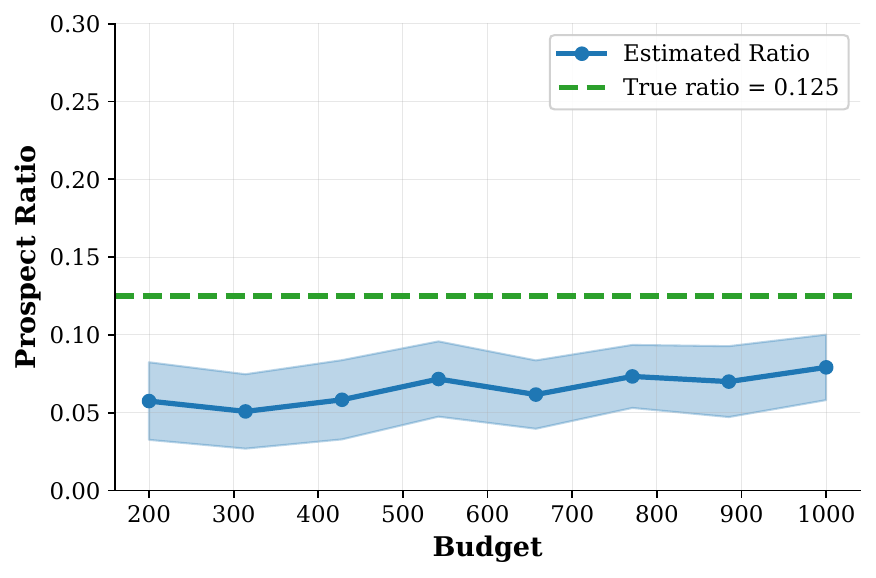}\hfill
\includegraphics[width=0.22\textwidth]{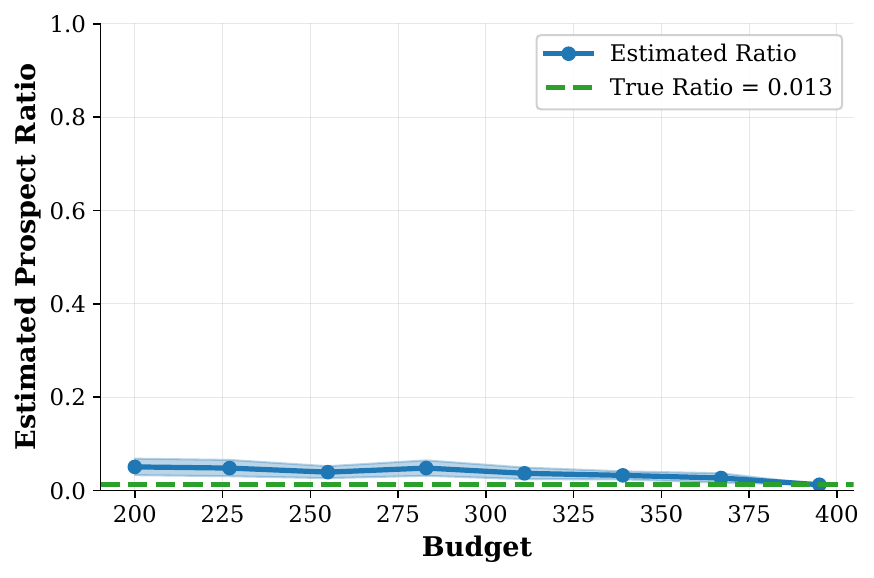}\hfill
\includegraphics[width=0.22\textwidth]{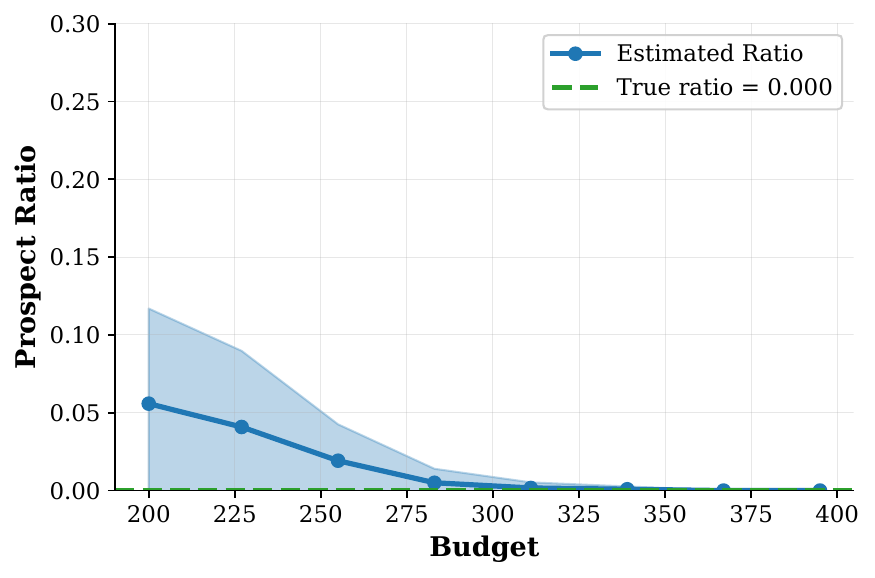}
\vspace{0.4cm} 
\includegraphics[width=0.22\textwidth]{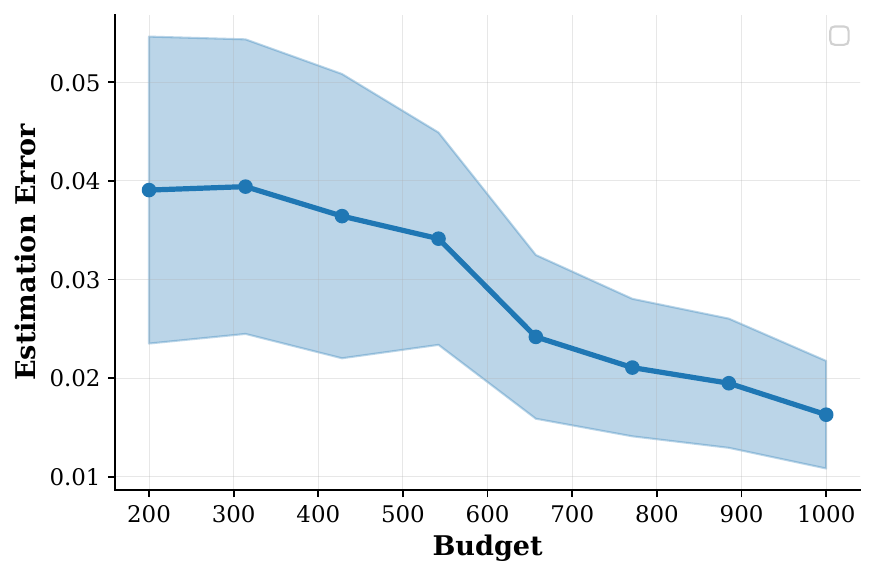}\hfill
\includegraphics[width=0.22\textwidth]{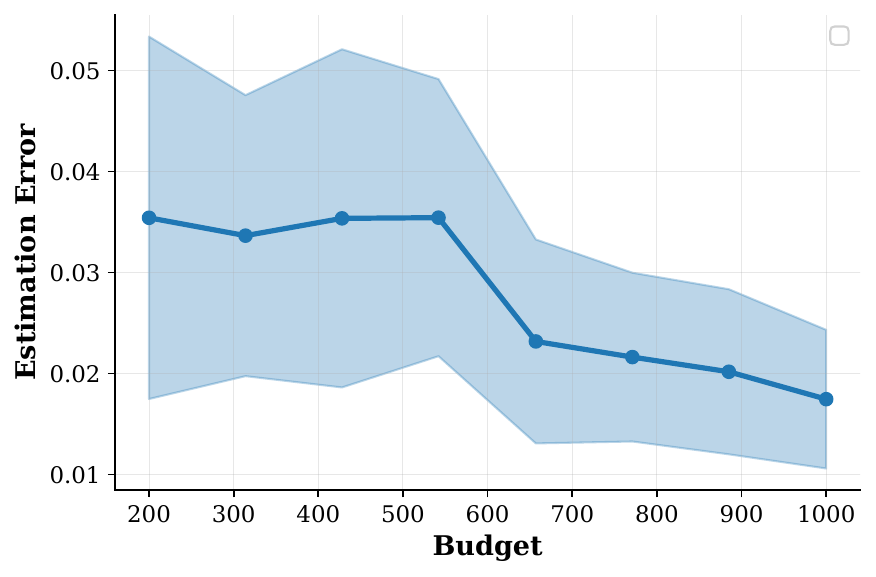}\hfill
\includegraphics[width=0.22\textwidth]{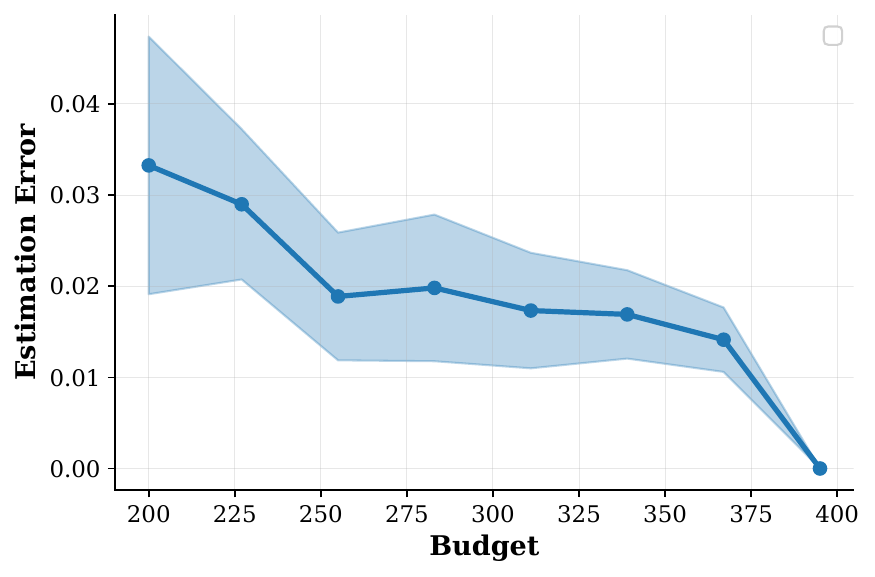}\hfill
\includegraphics[width=0.22\textwidth]{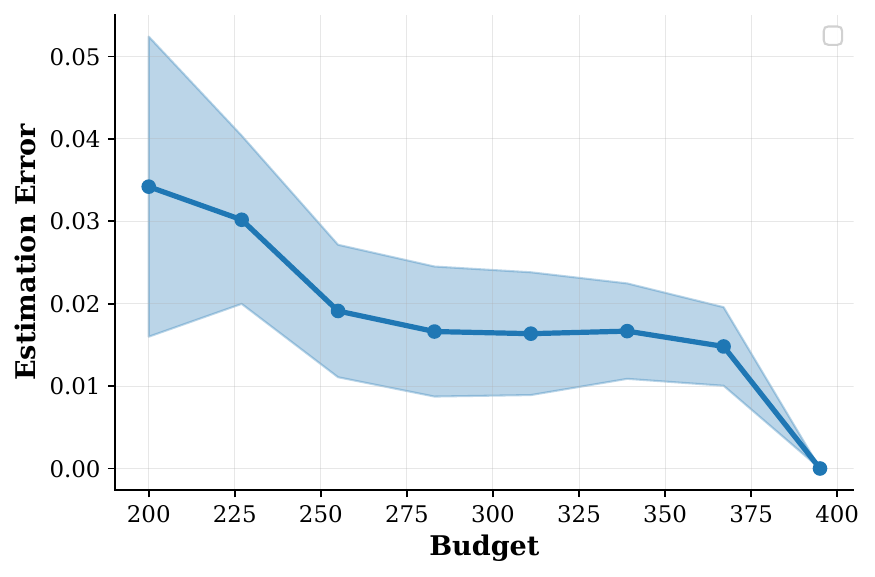}
\caption{Comparison of errors in statistical parity estimation and prospect ratio across different sample sizes.}\label{fig:experiments}
\end{figure*}

As discussed in Section~\ref{sec:stronglyrslts}, auditing an infinite hypothesis class is generally intractable. To address this, we adopt the sampling-based approximation described therein: we draw a fixed number of hypotheses to construct a finite representative subset. This subset preserves the geometric structure of the underlying strategic hypothesis class while enabling empirical evaluation. 

\noindent\subsection{Experimental Setup} 
Within this finite approximation, we evaluate \textsc{ERMP} (Algorithm~\ref{algo}) for strong auditability according to conditions (i) and (ii) in Definition~\ref{def:auditype2mp}. While \citet{yan2022active} evaluates the prospect class using its \emph{diameter},  $ \underset{f,f' \in \cP(\cF, \epsilon)}{\max} \mu(f) - \mu(f')$ which captures the worst-case spread of statistical parity values consistent with the observed data, a small diameter may arise simply because the prospect class contains only a single model. Such a scenario would trivially satisfy fairness but violate the \emph{completeness} requirement of strong auditing, namely that the auditor must consider a sufficiently rich set of compliant models.  To address this limitation, we complement the diameter-based analysis with an evaluation of condition (ii) using the \emph{estimated prospect ratio}, which better reflects the diversity of the prospect class beyond mere statistical error. 

We evaluate our approach on two standard fairness benchmarks: the \textsc{COMPAS} dataset \citep{angwin2016machine}, where groups are defined as Caucasian and non-Caucasian, and the \textsc{Student Performance} dataset, with groups defined as Female and Male. The true black-box model is a logistic regression classifier with $\ell_2$ regularization, trained on the original labels using scikit-learn’s default solver \citep{pedregosa2011scikit}. We consider two strategic model updates: one replacing the original model with a multi-layer perceptron (MLP), and another with a random forest. 

\subsection{Audit Fidelity Evaluation}
We assess the fidelity of our auditing procedure along two dimensions: 

(a) \textit{Prospect class comprehensiveness}: We measure the prospect ratio, the fraction of hypotheses that match the true model's statistical parity, and observe its convergence across increasing labeling budgets. As shown in the first row of Figure~\ref{fig:experiments}, the estimated ratio converges to the ground-truth ratio (computed on the full dataset), empirically validating that our algorithm reliably captures the true prospect class in the strong auditing setting.

(b) \textit{Prospect class correctness}: We evaluate whether models in the prospect class indeed exhibit statistical parity equivalent to the true model. The second row of Figure~\ref{fig:experiments} shows that the statistical parity estimation error for prospect models decreases with budget and remains consistently below our tolerance threshold $\epsilon = 0.005$, confirming the correctness of our prospect class construction.

\subsection{Runtime Evaluation} 

\textsc{EPO} oracle does not scales with the labeling budget, particularly for simpler strategic updates (i.e., random forests) requiring \textbf{$\sim$3 ms per sample} on average (see Table~\ref{tab:exp} and Figure~\ref{fig:bis}). This runtime accounts for both statistical parity estimation and prospect class construction. Additional details and hardware specifications are provided in Appendix~\ref{app:extexp}.

\begin{table}[ht]
\centering
\caption{Summary of experimental results for a budget of 1000 samples.}
\label{tab:exp}
\resizebox{\columnwidth}{!}{%
\begin{tabular}{lcccc}
\toprule
\textbf{Dataset} & \textbf{Strategic Class} & \textbf{Estimation Error} & \textbf{Ratio Error} & \textbf{Runtime (ms)} \\
\midrule
\textsc{COMPAS} & MLP & \num{2.33e-2} & \num{2.5e-3} & 1.9 \\
\textsc{COMPAS} & Random Forest & \num{2.33e-2} & \num{5e-3} & 1.3 \\
\textsc{Student Performance} & MLP & \num{8.22e-3} & 0 & 1.7 \\
\textsc{Student Performance}  & Random Forest & \num{8.22e-3} & \num{2.77e-17} & 1.2 \\
\bottomrule
\end{tabular}%
}
\end{table}
\section{Discussion and Future Work}\label{sec:directions}

We characterize the class of allowable strategic model updates by model owners --- namely, those that preserve the value of the audited property. We establish a necessary and sufficient condition in terms of the statistical parity (SP) dimension, a complexity measure strictly weaker than VC dimension and one that also appears in reconstruction-based property auditing using plug-in estimators. Our results suggest three natural directions for future work: (i) extending the framework to interactive settings via sequential, property-aware complexity measures; (ii) embedding audited properties directly into model architectures to obtain optimal, property-preserving predictors; and (iii) developing manipulation-proof audit definitions that are concept-class-agnostic and dimension-free.  While our focus was on traditional predictive models, our framework is particularly well-suited for auditing dynamic systems such as large language models (LLMs). LLMs undergo frequent updates that can shift their fairness behavior over time~\citep{chen2024chatgpt,schaeffer2023emergent}, and recent benchmarks reveal substantial subgroup disparities across models~\citep{parrish2021bbq}. Our estimator’s sample efficiency, stability under distribution shift, and ability to operate without full model access make it a promising candidate for continuous fairness monitoring in such settings, a direction we leave for future work.

\bibliographystyle{plainnat}
\bibliography{references}
\newpage 
\appendix

\section{Extended Related Works: Estimating Distributional Properties}\label{app:RW}

The auditing of properties of ML models has been explored in various contexts. For example, \cite{Bashir:2021} investigated model instability under feature imputation in a black-box setting, with guarantees on data minimization. More relevant to our work, several studies have focused on auditing group fairness to examine discriminatory behavior with respect to protected groups. In particular, \cite{kearns2018preventing} studied the auditing of statistical parity (SP) through a reduction to weak agnostic learning, where auditing is defined as verifying whether SP exceeds a given threshold. Building on this, \cite{hsu2024distribution} recently explored a specific case of this reduction, focusing on auditing SP under Gaussian feature distributions for homogeneous halfspace subgroups, and demonstrated the problem's computational hardness. Similarly, \cite{chugg2023auditing} examined the same verification problem within the statistical framework of hypothesis testing. More closely related to our work, \cite{yan2022active} examined auditing statistical parity by estimation instead of verification. In their approach, manipulation-proof constraints are enforced through active sampling. However, restricted to a finite hypothesis class, their method relies on reconstructing the model before plugging in the estimator. This limitation leads to the potential omission of hypotheses that could expand the size of the manipulation subclass, reducing the overall effectiveness of the auditing process.

In contrast to previous approaches, a variant of this problem has been investigated for auditing distributional properties, where the focus is on estimating characteristics of an unknown distribution. Examples of such properties includes Shannon entropy, the size of the distribution's support, support coverage, and various distribution metric distances (such as KL divergence and other divergence measures). Common approaches are often based on plug-in estimators that approximate the unknown distribution, typically requiring a logarithmic-factor increase in additional samples. Recent works by \cite{Hao:2018, Hao:2020} have explored this challenge within the realm of discrete distributions, incorporating smoothness assumptions and proposing an estimator that amplifies data relative to the empirical estimator by a factor  $\sqrt{\log n}$.  \cite{Hao:2020bis} extended the problem to the multi-distribution setting for estimating discrete distribution properties over a class of discrete distributions over $[k]$ by considering their mixtures (mixing strategies) and maintaining smoothness conditions with a sample complexity $\cO(\frac{k}{\epsilon^3 \sqrt{\log k}})$. \cite{Bharath:2012} explores integral probability metrics for continuous distributions, expanding upon various known distance metrics between distributions, such as total variance, Wasserstein distance, among others. \cite{Bharath:2012} further extended their analysis to kernelized distances, enriching the understanding of distance measures in the context of continuous distributions.
\section{Additional Auditing Scenarios}\label{app:examples}
In statistical learning in the agnostic setting, we study models for which we hope they behave the same as the data distribution. and we assume that the ground truth is fully described by the joint distribution $\cD$. This can be generalized to the setting where we learn models that have the same error over the ground truth.  This is equivalent to auditing the true error, where the true error is not necessarily the minimum over a hypothesis class.  The auditing framework extends the principles of statistical learning to a broader context, specifically that of black-box auditing of the model’s generalization error. In this setting, the regulator, denoted as $\mathfrak{R}$ (the auditor), operates in a black-box context where the internal structure of the learning algorithm (depicted by the left box in the figure) is inaccessible. The regulator $\mathfrak{R}$ is restricted to a single pass through data sampled according to the distribution $\mathcal{D}_{h^*}$ and must produce an audit hypothesis that approximates the true error of the black-box model $h^*$. While this black-box auditing of learning error scenarios may not have direct real-world applications, it serves to generalize the concept of agnostic learning, thereby extending it to the domain of agnostic auditing. This theoretical construct emphasizes the rigor of agnostic frameworks, illustrating their applicability to regulatory and auditing contexts where access to internal model parameters or structure is limited.

\paragraph{An analogy for PAC-audit: Bridging concepts.}

In the case of the statistical learning framework, given in the first example, the property $\mu$ represents the true error. In this case, the ground truth is zero (i.e. $\mu(\cD) = 0$), as the distribution itself encodes the truth. We say that the error of prediction and the learning error are coupled. In other words, for all $h \in \cH$, the (true) auditing error $|\mu(h) - \mu(\cD)|$ is equivalent to the true learning error $\cL_{\cD} (h) = \underset{(x,y) \sim \cD}{\prob}[h(x) \neq y]$. This is a special case where the property can be coupled to $|\mu(h) - \mu(\cD)| = \mu(h, \cD)$.

In summary, classic PAC learning produces a hypothesis that minimizes error and provides an estimate of the true error. This specific case, which outputs a single hypothesis, forms a one-sized, manipulation-proof subclass. Similarly, PAC auditing generalizes this framework by replacing error with a broader, abstract distributional property.

We now present examples of property auditing, demonstrating that weakly auditing naturally reduces to a risk minimization framework analogous to that in statistical theory.


\subsection{Loss functions}

\begin{definition}
    For any auditing problem, let $\{\cX_i\}_{i \in I}$ denote a feature partition of $\cX$, $\mu$ a distributional property defined over this partition, and $h$ the outputted model by the auditor.
    \begin{itemize}
    \item The auditing loss is defined as $ \ell_{\mu}: (\cY \times \cY)^{|I|} \longrightarrow \left[0,1 \right]$.
    \item The true auditing risk is defined as:
    \begin{align*}
        \cE_{\mu}(f) \triangleq \underset{(X,Y) \sim \cD}{\E} [\ell_{\mu}((f(X_i), Y_i)_{i \in I}) | X_i \in \cX_i]
    \end{align*}
    \item Let $S = \bigcup_{i \in I} S_i$ denotes a set of instances sampled i.i.d from $\cD$, where $S_{i_{|\cX}} \subseteq \cX_i$. For all $i \in I$, let $m_i$ denote the cardinal of $S_i$ and $S_i = \{(x^j_i,y^j_i)\}_{j=1}^{m_i}$. Let $m = \prod_{i \in I} m_i$.   The empirical auditing risk with respect to $S$ is
    \begin{align*}
        \hat{\cE}_{\mu}(f) \triangleq \frac{1}{m} \sum_{j = 1}^m  \ell_{\mu}((f(x^j_i),y^j_i)_{i \in I})
    \end{align*} 
    \end{itemize}
\end{definition}

The proposed general loss function extends the PAC learning framework, traditionally used to learn or audit properties defined over a single population, to settings where the target property depends on multiple populations. 

Table~\ref{tab:proploss} summarizes the auditing losses corresponding to these properties. 

\setlength{\textfloatsep}{6pt}
\begin{table}[t!]
    \centering
    \resizebox{0.8\columnwidth}{!}{
    \begin{tabular}{c|c}
        Property & Loss function \\
        \midrule
        Statistical Parity & $|\frac{1}{m_0} \mathds{1}_{h(x_i) = 1|x_i \cX_0} - \frac{1}{m_1} \mathds{1}_{h(x_i') = 1|x_i' \cX_1}|$\\
        Learning error & $|\mathds{1}_{h(X) \neq Y} - \mathds{1}_{Y^* \neq Y}|$\\
        Generalization error & $ \bigl| \big(\mathds{1}_{h(x) \neq y} - \mathds{1}_{h(\Tilde{x}) \neq \Tilde{y}} \big) - \big(\mathds{1}_{h^*(x) \neq y} - \mathds{1}_{h^*(\Tilde{x}) \neq \Tilde{y}}\big) \bigl|  $ \\
          Robust risk & ${\sup_{z \in \cU(x)}}|\mathds{1}_{h(z) \neq y} - \mathds{1}_{y^* \neq y}|$ \\
        \bottomrule
    \end{tabular}}
    \caption{Loss functions for auditing different properties.}\label{tab:proploss}
\end{table}

\subsection{Learning Error.} 

For auditing the learning error, the auditing risk is defined as: 
    \begin{align*}
         \cE(h) = \underset{\substack{(X,Y) \sim \cD \\ (X,Y^*) \sim \cD^* }}{\E} [\ell((h(X), Y), (Y^*, Y))]
    \end{align*}

Where the loss is defined as:

\begin{align*}
         \ell((h(X), Y), (Y^*, Y)) = |\mathds{1}_{h(X) \neq Y} - \mathds{1}_{Y^* \neq Y}|
    \end{align*}

\begin{proposition}
    Risk minimization implies weakly auditing of learning error.
\end{proposition}

\begin{proof}
    Let $(\epsilon, \delta) \in (0,1)^2$, and let $h \in \cH$,
    
    Based on the definition of audit risk and applying Jensen's inequality, we obtain:

    \begin{align*}
        \cE(h, \mu) &= \underset{\substack{(X,Y) \sim \cD \\ (X,Y^*) \sim \cD^* }}{\E} [\ell((h(X), Y), (Y^*, Y))]\\
        &= \underset{\substack{(X,Y) \sim \cD \\ (X,Y^*) \sim \cD^* }}{\E} [|\mathds{1}_{h(X) \neq Y} - \mathds{1}_{Y^* \neq Y}|] \\
        & \geq \Big| \underset{\substack{(X,Y) \sim \cD \\ (X,Y^*) \sim \cD^* }}{\E} [\mathds{1}_{h(X) \neq Y} - \mathds{1}_{Y^* \neq Y}]  \Big| \\
        &= |\mu(h) - \mu^*|
    \end{align*}

    Hence, 

    \begin{align*}
        \underset{S \sim \cD^m}{\prob} \Big[  |\mu(h) - \mu^*| > \epsilon  \Big] \leq \underset{S \sim \cD^m}{\prob} \Big[\cE_m(h, \mu) > \epsilon \Big]
    \end{align*}

    C/C:

    \begin{align*}
        \forall (\epsilon, \delta) \in (0,1)^2, \forall h \in \cH: \underset{S \sim \cD^m}{\prob} \Big[\cE_{\mu}(h) > \epsilon \Big] \leq \delta \implies \underset{S \sim \cD^m}{\prob} \Big[  |\mu(h) - \mu^*| > \epsilon  \Big] \leq \delta
    \end{align*}
    
\end{proof}

    





    

\subsection{Generalization Error}




The auditing risk for generalization error is

\begin{equation*}
\cE(h) = \underset{\substack{(x,y) \sim \cD_{\text{train}}  (\Tilde{x},\Tilde{y}) \sim \cD_{\text{test}}}}{\E} \Big[ \ell\biggl(\big((h(x),y), (h(\Tilde{x}), \Tilde{y})   \big) ,  \big((h^*(x),y) , (h^*(\Tilde{x}), \Tilde{y}) \big)   \biggl) \Big] 
\end{equation*}
Here, $\ell\biggl(\big((h(x),y), (h(\Tilde{x}), \Tilde{y})   \big) , \big((h^*(x),y) , (h^*(\Tilde{x}), \Tilde{y}) \big)   \biggl) =  \biggl| \big(\mathds{1}_{h(x) \neq y} - \mathds{1}_{h(\Tilde{x}) \neq \Tilde{y}} \big) - \big(\mathds{1}_{h^*(x) \neq y} - \mathds{1}_{h^*(\Tilde{x}) \neq \Tilde{y}}\big) \Biggl|  $.
\begin{proposition}
    Risk minimization implies weakly auditing of generalization error.
\end{proposition}

    \begin{proof}
      Let $(\epsilon, \delta) \in (0,1)^2$, and let $h \in \cH$,

    Based on the definition of audit risk, we obtain:

    \begin{align*}
        \cE(h) &= \underset{\substack{(x,y) \sim \cD_{\text{train}} \\ (\Tilde{x},\Tilde{y}) \sim \cD_{\text{test}}}}{\E} \Big[ \ell\biggl(\big((h(x),y), (h(\Tilde{x}), \Tilde{y})   \big) , \big((h^*(x),y) , (h^*(\Tilde{x}), \Tilde{y}) \big)   \biggl) \Big] \\
        & = \underset{\substack{(x,y) \sim \cD_{\text{train}} \\ (\Tilde{x},\Tilde{y}) \sim \cD_{\text{test}}}}{\E} \Big[\biggl| \big(\mathds{1}_{h(x) \neq y} - \mathds{1}_{h(\Tilde{x}) \neq \Tilde{y}} \big) - \big(\mathds{1}_{h^*(x) \neq y} - \mathds{1}_{h^*(\Tilde{x}) \neq \Tilde{y}}\big) \Biggl| \Big] \\
         & \geq  \biggl| \underset{\substack{(x,y) \sim \cD_{\text{train}} \\ (\Tilde{x},\Tilde{y}) \sim \cD_{\text{test}}}}{\E} \Big[ \big(\mathds{1}_{h(x) \neq y} - \mathds{1}_{h(\Tilde{x}) \neq \Tilde{y}} \big) - \big(\mathds{1}_{h^*(x) \neq y} - \mathds{1}_{h^*(\Tilde{x}) \neq \Tilde{y}}\big)\Big] \Biggl| \\
        &= |\mu(h) - \mu^*|
    \end{align*}

    Where in the third step, we use Jensen inequality.

    Hence, 

    \begin{align*}
        \underset{S \sim \cD^m}{\prob} \Big[  |\mu(h) - \mu^*| > \epsilon  \Big] \leq \underset{S \sim \cD^m}{\prob} \Big[\cE_{\mu}(h) > \epsilon \Big]
    \end{align*}

    C/C:

    \begin{align*}
        \forall (\epsilon, \delta) \in (0,1)^2, \forall h \in \cH: \underset{S \sim \cD^m}{\prob} \Big[\cE_{\mu}(h) > \epsilon \Big] \leq \delta \implies \underset{S \sim \cD^m}{\prob} \Big[  |\mu(h) - \mu^*| > \epsilon  \Big] \leq \delta
    \end{align*}
\end{proof}

\subsection{Robust Risk} 
Similarly to previous work on auditing learning errors, we extend the robust learning framework to the problem of auditing robust risk with respect to arbitrary perturbation sets. The auditing loss is:
    \begin{align*}
        \ell_{\cU}((h(\cU(x)),y),(h^*(\cU(x)), y)) = \underset{z \in \cU(x)}{\sup}|\mathds{1}_{h(z) \neq y} - \mathds{1}_{y^* \neq y}|
    \end{align*}
    
And the corresponding true auditing risk is given by:

   \begin{align*}
         \cE_ {\cU}(h) = \underset{(x, y ) \sim \cD }{\E} [\ell_{\cU}((h(\cU(x)), y), (h^*(\cU(x)), y)) ]
    \end{align*}

\begin{proposition}
 For any set of perturbations $\cU$, risk minimization implies weakly auditing robust risk.
\end{proposition}

An intriguing question arises: which hypothesis classes $\cH$ are robustly auditable with respect to an arbitrary perturbation set $\cU$? The capacity of these classes depends on the choice of the perturbation set. 

    \begin{conjecture}
       There exists an $(\epsilon, \delta)$- weakly auditor for the robust risk auditing problem $\langle \cX, \cY, \cH,\cP, \mu_{\cU}, \ell_{\cU}\rangle$ if and only if a complexity measure $\mathfrak{D}_{\cU}(\cH)$ if finite.
    \end{conjecture}

\section{Auditing with Prospects: Connections to Rashomon Sets}\label{app:rashomon}

Recent work has studied the complexity of learning "simple" models, which carries nuanced meanings across ML communities. For example, in healthcare applications, simplicity often refers to developing explainable models that address the black-box inexplainability problem \cite{rudin2019stop}. Alternatively, in fairness-aware machine learning, simplicity can mean identifying fair models within an equally accurate model class \cite{agarwal2018reductions}. This phenomenon, known as the Rashomon effect, aims to explore multiple perspectives of the joint dataset distribution revealing different truths. The Rashomon set represents a collection of models that achieve comparable performance but differ in their explanations or underlying patterns. The Rashomon ratio, introduced by \citep{semenova2022existence}, quantifies this effect by measuring the volume of the Rashomon class relative to the hypothesis class.

\begin{definition}[Rashomon set]
    Given $\epsilon > 0$, a dataset $S$, a hypothesis class $\cH$ and a loss function $\ell$,
    \begin{enumerate}
        \item The true Rashomon set is:
        $$\cR_{\cD}(\cH, \epsilon) \triangleq \{h \in \cH: \cL_{\cD}(h) \leq \underset{h' \in \cH}{\argmin}\cL_{\cD}(h') + \epsilon \} $$
        \item The empirical Rashomon set is:
        $$\hat{\cR}_{S}(\cH, \epsilon) \triangleq \{h \in \cH: \hat{\cL_S}(h) \leq \underset{h' \in \cH}{\argmin} \hat{\cL}_S(h') + \epsilon \} $$
    \end{enumerate}

    where $\cL_{\cD}$ and $\hat{L}_S$ are the true and empirical risk respectively, defined as $\cL_{\cD} = \underset{(x,y) \cD}{\E} \{l(h,(x,y))\}$, $\hat{L}_S = \frac{1}{|S|} \underset{(x,y) \in S}{\sum} \ell(h,(x,y))$.
    
\end{definition}

While Rashomon sets and prospect classes both explore sets of models that maintain specified performance properties, such as bounded learning error or auditing error, their theoretical foundations differ significantly. Rashomon sets have primarily been studied in the context of finite hypothesis classes \cite{semenova2022existence}, where the central theoretical challenge is not the information-theoretic complexity of learning \footnote{as this is already characterized by the VC dimension.}, but rather quantifying the relative simplicity of learning with respect to the empirical data distribution. This simplicity is formally captured by the Rashomon ratio, which measures the proportion of models achieving a specified performance threshold. Consequently, practitioners typically focus on selecting a single 'simple' model from the Rashomon set, rather than characterizing its complete structure or geometric properties. In contrast, prospect classes, central to manipulation-proof auditing theory, require explicit characterization of all models that satisfy the auditing criteria. This comprehensive enumeration requirement arises from the need to reason about all possible ways model owner might manipulate the pre-audit model while maintaining acceptable performance. 

The distinction highlights that prospect classes and Rashomon sets, despite their structural similarities in characterizing sets of property preserving models, serve fundamentally different theoretical objectives: while Rashomon sets focus on quantifying model simplicity within finite hypothesis classes via the Rashomon ratio, prospect classes are concerned with the information-theoretic complexity of learning the entire set of property preserving models in potentially infinite hypothesis spaces.

\section{Proof of Strategic Lemma}\label{app:strategiclemma}
Before proceeding with the proof, we restate Lemma \ref{SP-strategiclemma} for clarity:

\strategiclemma*

\begin{proof}
Let $ \cD$ be a distribution on $\cX \times \cY$ and  $S$ be a set of samples sampled from $\cD$, and   $h_S \in \cA(S)$,

By triangle inequality,

\begin{align*}
    | \mu_{\cD} (h_S) - \text{opt}(\cD, \cH)| \leq | \mu_{\cD} (h_S) - \hat{\mu}_S(h_S)| + | \hat{\mu}_S (h_S) - \text{opt}(\cD, \cH)| + | \text{opt}(\cD, \cH) - \text{opt}(\cD, \cH)|
\end{align*}

This inequality is verified for any $S \sim D^n$, hence, 

\begin{align*}
    \prob_{S \sim \cD^n} \Bigl[| \mu_{\cD} (h_S) - \text{opt}(\cD, \cH)| \leq \epsilon \Bigl] & \geq  \prob_{S \sim \cD^n} \Bigl[| \mu_{\cD} (h_S) - \hat{\mu}_S(h_S)| \leq \frac{\epsilon}{3}\Bigl]\\
    &\quad + \prob_{S \sim \cD^n} \Bigl[| \hat{\mu}_S (h_S) - \text{opt}(\cD, \cH)| \leq \frac{\epsilon}{3}\Bigl]  \\
    & \quad  + \prob_{S \sim \cD^n} \Bigl[| \text{opt}(\cD, \cH) - \text{opt}(\cD, \cH)| \leq \frac{\epsilon}{3}\Bigl] - 2\\
\end{align*}

Equivalently,

\begin{align*}
    \prob_{S \sim \cD^n} \Bigl[| \mu_{\cD} (h_S) - \text{opt}(\cD, \cH)| > \epsilon \Bigl] & \leq  \prob_{S \sim \cD^n} \Bigl[| \mu_{\cD} (h_S) - \hat{\mu}_S(h_S)| > \frac{\epsilon}{3}\Bigl] \\
    \quad + \prob_{S \sim \cD^n} \Bigl[| \hat{\mu}_S (h_S) - \text{opt}(\cD, \cH)| > \frac{\epsilon}{3}\Bigl] + \\
    & \quad \prob_{S \sim \cD^n} \Bigl[| \text{opt}(\cD, \cH) - \text{opt}(\cD, \cH)| > \frac{\epsilon}{3}\Bigl] \\
\end{align*}

On the other hand, 

\begin{equation*}
\begin{cases}
    | \mu_{\cD} (h_S) - \hat{\mu}_S(h_S)| > \frac{\epsilon}{3} & \implies \exists h \in \cH, | \mu_{\cD} (h) - \hat{\mu}_S(h)| > \frac{\epsilon}{3} \\
    | \text{opt}(\cD, \cH) - \hat{\text{opt}}(S, \cH)| > \frac{\epsilon}{3} & \implies \exists h \in \cH, | \mu_{\cD} (h) - \hat{\mu}_S(h)| > \frac{\epsilon}{3} \\
\end{cases}
\end{equation*}

By using SP-uniform convergence and sample-dependent manipulation proof,

\begin{align*}
        \prob_{S \sim \cD^n} \Bigl[| \mu_{\cD} (h_S) - \text{opt}(\cD, \cH)| > \epsilon \Bigl] & <  \prob_{S \sim \cD^n} \Bigl[| \mu_{\cD} (h_S) - \hat{\mu}_S(h_S)| > \frac{\epsilon}{3}\Bigl] + \prob_{S \sim \cD^n} \Bigl[| \hat{\mu}_S (h_S) - \text{opt}(\cD, \cH)| > \frac{\epsilon}{3}\Bigl] \\
        &< \delta
\end{align*}
\end{proof}
\section{Proofs for Weakly Auditable Classes}\label{sec:proofs_sp_mp_type1} 

\subsection{All Finite Classes are Weakly Auditable (Proof of Theorem \ref{SPauditfiniteH})}\label{prooflem1}
We begin by restating Theorem \ref{SPauditfiniteH}:
\theoremweakfinite*
\begin{proof}
For $i \in \{0,1\}$, let $m_i$ denote the sample size of the $i$-th protected group.

Since $\cF$ is finite, it is sufficient to show the SP-uniform convergence, and deduce the result for the SP-strategic lemma.

From Lemma~\ref{lemma:discchernoff}, we have 

By triangle inequality, we have for all $h \in \cF$:

\begin{align} \label{eq1proof2}
    \prob_{\cD^n} \Bigl[|\mu_{\cD}(h) - \hat{\mu}_S(h)| \geq \frac{\epsilon}{3}  \Bigl] \leq \exp \{- \frac{m_0 m_1 \epsilon^2}{18(m_0+m_1)}\}
\end{align}

By using the union bound over finite hypothesis class $\cF$, and Claim~\ref{claiminterprob}, the right part in inequality \ref{eq1proof2} is upper bounded by $\delta$ if the sample complexity $m$ verifies $m= \cO \Bigl( \Bigl\lceil \frac{18}{ \epsilon^2} \log \frac{8 |\cF|}{\delta} \Bigl \rceil  \Bigl)$

\end{proof}

\subsection{Weakly Auditability Does Not extend to All Infinite VC Classes}\label{sec:agnostic_sp_mp_type1}
First, we prove a technical lemma.

\begin{lemma}[Lemma~\ref{lemma:welldef}]
    For any finite subsets $S_0$ and $S_1$ drawn from the first and second protected groups respectively, the set $\Delta_{\cC}^{SP}(S_0, S_1)$ is well defined.
\end{lemma}

\begin{proof}
    To establish the result in Lemma~\ref{lemma:welldef}, it is sufficient to show that, for any concept class $\cC$ of $\text{VC}(\cF) = d$, and for any shattered set $S = \{x_1, \cdots, x_d\}$ shattered by $\cF$, not all elements of $S$ can belong to the same protected group.

    By the definition of VC dimension, the total number of dichotomies for the sample $S$ is $2^d$.
    We assume by contradiction that all elements of $S$ belong to a single protected group. Without loss of generality, we assume that $S \subseteq \cX_0$. Let $x'$ be any point from $\cX_1$.
    
    For all $i$ in $[2^d]$, let $c_i$ denote the concept from $\cC$ that realizes the dichotomy $u_i$. Since $\cX_0$ and $\cX_1$ form the set of components of $\cX$ ($\cX_0 \cap \cX_1 = \emptyset$), each $c_i$ can be extended to $\Tilde{c^0_i}$ and $\Tilde{c^1_i}$, such that $x' \in \Tilde{c^0_i}$ and $x' \notin \Tilde{c^1_i}$. 

    Hence, $\cC$ realizes $2^{d+1}$ dichotomies over $S \cup \{x'\}$. In other words, $S \cup \{x'\}$ shatters $\cC$. This is a contradiction because $\text{VC}(\cC) =d$.
\end{proof}

\subsubsection{Upper Bound on Sample Complexity}
We maintain the same measurability conditions from PAC learning, that is, $\cF$ is indexed by Borel sets of a finite-dimensional space (Appendix C in \cite{pollard2012convergence} and Section 5 in \cite{ben1990learning}). 
This assumption is important to derive upper bounds as operations using the concept class require this assumption. 

\paragraph{A note on the measure of sample complexity.}  Our approach to sample complexity differs from classical methods fundamentally. While traditional bounds often consider a single data source, our group fairness audit framework addresses the challenge of information flowing from two distinct protected groups. In this context, we propose that the minimum number of samples required from each group serves as the natural measure of sample complexity for accurate statistical parity estimation. An important practical implication follows: model owners may face manipulations that adjust the relative proportions between groups. This choice of sample complexity thereby manages the problem of class imbalance with considerable flexibility, provided that the sample size for each group never falls below this established threshold.

Let $\Big(\cX \times \cY, \Omega(\cX \times \cY), \cD \Big)$ denote a probabilistic space, where $\cX$ can be uncountable. Let $\cF$ denote the set of all functions defined from $\cX$ to $\cY$.

And let $\cC$ denote a concept class, 
\begin{align*}
    \cC = \Bigl \{c: c \subseteq \cX, \exists \mathds{1}_c \in \cF, x \in c \iff \mathds{1}_c(x) = 1  \Bigl \}
\end{align*}

For a sample $S$ of size $2m$, we use the notation $S^{\triangleright}$ to denote the first $m$ instances in sample $S$, and $S^{\triangleleft}$ to denote the remaining $m$ instances in sample $S$ (i.e. $S = \langle S^{\triangleright}, S^{\triangleleft} \rangle$). 

We first prove the following theorem:

\begin{theorem}\label{theorem:distdepL}
    Let $\epsilon \in (0,1)$ and $\cC$ a concept class of SP dimension $\mathfrak{p}$. The probability that the empirical SP-auditing risk of at least one concept differs from the true SP-auditing risk by more than $\epsilon$ in an i.i.d sample $S \sim \cD$ of size $m_0 +m_1$ with $\min(m_0, m_1) \geq \frac{4}{\epsilon^2} \log\frac{1}{1-\alpha}$ satisfies the following inequality:

    \begin{align*}
        \underset{S \sim \cD^m}{\prob} \Big[\sup_{c \in \cC} |\mu_{\cD}(c) - \mu_S(c)| > \epsilon \Big]  \leq \alpha \Pi^{\SP}_{\mathfrak{p}}(2 m_0, 2m_1) \exp \frac{- \min(m_0,m_1)}{16}\epsilon^2 
    \end{align*}
\end{theorem}

As we will see in the proof, $\alpha$ can be any constant in $(0,1)$ depending on the problem parameters. By proving this result, the proof follows from Claim \ref{claim:separation}. 

We begin by proving the following lemma:

\begin{lemma}\label{lemma:sym1}
    For all $\alpha \in (0,1)$, $\min(m_0,m_1)  \geq \frac{4}{\epsilon^2} \log\frac{1}{1-\alpha}$,
    \begin{align*}
        \underset{S \sim \cD^m}{\prob} \Big[\sup_{c \in \cC}  |\mu_{\cD}(c) - \mu_S(c)| > \epsilon \Big] \leq \alpha \underset{S \sim \cD^{2m}}{\prob} \Big[\sup_{c \in \cC} |\mu_{S^{\triangleright}}(c) - \mu_{S^{\triangleleft}}(c)| > \frac{\epsilon}{2} \Big]
    \end{align*}
\end{lemma}

\textbf{Interpretation of the result:} For a sample size of $\frac{4}{\epsilon^2} \log\frac{1}{1-\alpha}$, if we repeat the experiment of measuring the random quantity of empirical statistical parity twice, the probability that the outcomes deviate by at least half of epsilon between the two experiments provides a lower bound on the probability that the empirical estimate deviates by more than epsilon from the true value of statistical parity. This provides the first step to prove uniform convergence. Now we proceed with the proof:

\begin{proof}

Let $\cK_{\epsilon,2m}$ denote the event\footnote{The event in the right side of the inequality in Lemma~\ref{lemma:sym1}.}:

\begin{align*}
    \cK_{\epsilon,2m} \triangleq \Big \{S: S \subseteq (\cX \times \cY)^{2m},  \sup_{c \in \cC} |\mu_{S^{\triangleright}}(c) - \mu_{S^{\triangleleft}}(c)| > \frac{\epsilon}{2} \Big \}
\end{align*}
    By the definition of expectation, we have,

    \begin{align*}
        \underset{S \sim \cD^{2m}}{\prob}(\cK_{\epsilon,2m}) = \int_{S \in (\cX \times \cY)^{2m}} \mathds{1}_{\cK_{\epsilon,m}} (S) \,d \cD^{2m}(S)
    \end{align*}

    By Fubini's theorem, 

     \begin{align*}
       \underset{S \sim \cD^{2m}}{\prob}(\cK_{\epsilon,2m})  = \int_{S^{\triangleright} \in (\cX \times \cY)^m} \int_{S^{\triangleleft} \in (\cX \times \cY)^m} \mathds{1}_{\cK_{\epsilon,m}} (\langle S^{\triangleright}, S^{\triangleleft}\rangle) \,d \cD^m(S^{\triangleright}) \,d \cD^m(S^{\triangleleft})
    \end{align*}

Let $\cJ_{\epsilon,m}$ denote the event we seek to establish an upper bound on its probability:

\begin{align*}
    \cJ_{\epsilon,m} \triangleq \Big \{S: S \subseteq (\cX \times \cY)^m,  \sup_{c \in \cC} |\mu_{\cD}(c) - \mu_S(c)| > \epsilon \Big \}
\end{align*}

    Since $\cJ_{\epsilon,m} \subseteq (\cX \times \cY)^m$, we obtain:

     \begin{align}\label{ineq:2int}
        \underset{S \sim \cD^{2m}}{\prob} (\cK_{\epsilon,2m}) \geq \int_{S^{\triangleright} \in \cJ_{\epsilon,m}} \int_{S^{\triangleleft} \in (\cX \times \cY)^m} \mathds{1}_{\cK_{\epsilon,m}} (\langle S^{\triangleright}, S^{\triangleleft} \rangle) \,d \cD^m(S^{\triangleright}) \,d \cD^m(S^{\triangleleft})
    \end{align}

By definition of $\cJ_{\epsilon,m}$, for every $S^{\triangleright}$ in $\cJ_{\epsilon,m}$, there exists a concept $c_{S^{\triangleleft}}$ that verifies the following inequality:

\begin{align*}
     |\mu_{\cD}(c_{S^{\triangleright}}) - \mu_{S^{\triangleright}}(c_{S^{\triangleright}})| > \epsilon
\end{align*}

On the other hand, by the inverse triangle inequality, we have for any $S^{\triangleright}$ in $\cJ_{\epsilon,m}$ and any $S^{\triangleleft}$ in $(\cX \times \cY)^m$:

\begin{align*}
    |\mu_{S^{\triangleright}}(\mu_{S^{\triangleright}}) - \mu_{S^{\triangleleft}}(c_{S^{\triangleright}})| \geq |\mu_{S^{\triangleright}}(c_{S^{\triangleright}}) - \mu_{\cD}(c_{S^{\triangleright}})| - |\mu_{S^{\triangleleft}}(c_{S^{\triangleright}}) - \mu_{\cD}(c_{S^{\triangleright}})|
\end{align*}

We deduce that for any $S^{\triangleleft}$ in $(\cX \times \cY)^m$ a sufficient condition to have $\langle S^{\triangleright}, S^{\triangleleft}\rangle \in \cK_{\epsilon,2m}$ is $|\mu_{S^{\triangleleft}}(c_{S^{\triangleright}}) - \mu_{\cD}(c_{S^{\triangleright}})| \leq \frac{\epsilon}{2}$.
Let $\cA_{S^{\triangleright}}$ denote the set of these events:

\begin{align*}
    \cA_{S^{\triangleright}} \triangleq \Big\{S^{\triangleleft}: S^{\triangleleft} \subseteq (\cX \times \cY)^m, |\mu_{S^{\triangleleft}}(c_{S^{\triangleright}}) - \mu_{\cD}(c_{S^{\triangleright}})| \leq \frac{\epsilon}{2}   \Big\}
\end{align*}

We have shown that: 

\begin{align*}
    S^{\triangleleft} \in \cA_{S^{\triangleright}} \implies \langle S^{\triangleright}, S^{\triangleleft}\rangle \in \cK_{\epsilon,2m}
\end{align*}

By implementing this in inequality \ref{ineq:2int}, we obtain:

\begin{align}\label{ineq:onesample}
            \underset{S \sim \cD^{2m}}{\prob} (\cK_{\epsilon,2m})  \geq \int_{S^{\triangleright} \in \cJ_{\epsilon,m}} \int_{S^{\triangleleft} \in (\cX \times \cY)^m} \mathds{1}_{\cA_{S^{\triangleright}}} ( S^{\triangleleft}) \,d \cD^m(S^{\triangleright}) \,d \cD^m(S^{\triangleleft})
\end{align}

By using the result in Lemma~\ref{lemma:discchernoff}, we have for all $S^{\triangleright}$ in $\cX^m$ :

\begin{align*}
    \underset{S^{\triangleleft} \sim \cD^m}{\prob} \Big[|\mu_{S^{\triangleleft}}(c_{S^{\triangleright}}) - \mu_{\cD}(c_{S^{\triangleright}})| \leq \frac{\epsilon}{2}   \Big]  &\geq 1 - \exp \Big\{\frac{-  m_0 m_1}{2(m_0+m_1)} \epsilon^2 \Big\} \\
    & \geq 1 - \exp \Big \{ - \frac{\min(m_0, m_1) \epsilon^2}{4} \Big\}
\end{align*}

Where the second step follows from Claim~\ref{claiminterprob}. The right side is bigger than $\alpha$ for $\min(m_0,m_1) \geq \frac{1}{\epsilon^2} \log \frac{1}{1 - \alpha}$. We have shown that for any $m$ such that, $\min(m_0,m_1) \geq \frac{4}{\epsilon^2} \log \frac{1}{1 - \alpha}$, and any $S^{\triangleright} \in \cJ_{\epsilon,m}$:

\begin{align*}
    \int_{S^{\triangleleft} \in (\cX \times \cY)^m} \mathds{1}_{\cA_{S^{\triangleright}}} (S^{\triangleleft}) \,d \cD^m(S^{\triangleleft}) \geq \alpha
\end{align*}

By implementing this in inequality \ref{ineq:onesample}, we obtain the desired result.

\textbf{Remark: Measurability assumption.} We can observe that $\cJ_{\epsilon,m}$  and $\cK_{\epsilon,2m}$ are measurable events when $\cC$ is countable or finite. In the general case, even when $\cC$ is measurable it does not imply the events are measurable. A relaxed assumption of measurability is assuming $\cC$ is indexed by a collection of Borel sets in an Euclidean space \citep{pollard2012convergence,ben1990learning}.

\end{proof}

\begin{lemma}\label{lemma:spsym2}
    \begin{align*}
        \underset{S \sim \cD^{2m}}{\prob} (\cK_{\epsilon,2m}) \leq \Pi^{SP}_{\mathfrak{p}}(2 m_0, 2m_1) \exp \frac{- \min(m_0,m_1)}{16}\epsilon^2 
    \end{align*}
\end{lemma}

WHere $\mathfrak{p} = \SP(\cC)$

\begin{proof}
    Let $\Pi_{m_0,m_1}$ the symmetric group defined on $[2 m_0] \times [ 2 m_1]$:

    \begin{align*}
        \Pi_{m_0,m_1} \triangleq \Big \{ (\pi_0, \pi_1) \in \mathfrak{S}_{2m}^2: &\forall (i,j) \in [2 m_0] \times [2 m_1], \\
        & (\pi_0(i) = i \land \pi_0(m_0+i) = m_0+i) \lor  (\pi_0(i) = m_0+i \land \pi_0(m_0+i) = i), \\
        & \quad (\pi_1(i) = i \land \pi_1(m_1+i) = m_1+i) \lor  (\pi_1(i) = m_1+i \land \pi_1(m_1+i) = i)   \Big \}
    \end{align*}

    We have, $|\Pi_{2m}| = 2^{m_0 + m_1} = 2^m$.

    By the i.i.d assumption on sampling from $\cD$, and given the definition of the permutation set (where the permutation acts independently over each protected group), for any $\pi \in \Pi_{2m}$ such that $\pi = (\pi_0, \pi_1)$:

    \begin{align*}
        \underset{S \sim \cD^{2m}}{\prob}(\pi(S)) &= \prod_{i=1}^{2m} \underset{\cD}{\prob}\Big( (x_{\pi_0(i)}, y_{\pi_1(i)}) \Big)\\
        &= \prod_{i=1}^{2m} \underset{\cD}{\prob} \Big( (x_i,y_i) \Big) \\
        & = \underset{\cD^{2m}}{\prob}(S)
    \end{align*}

Hence, for every permutation $\pi \in \Pi_{2m}$, such that $\pi = (\pi_0, \pi_1)$:

\begin{align*}
     \underset{S \sim \cD^{2m}}{\prob} (\cK_{\epsilon,2m}) &=  \int_{S \in (\cX \times \cY)^{2m}} \mathds{1}_{\cK_{\epsilon,2m}} (\pi(S)) \,d \cD^{2m}(S) \\
     &= \frac{1}{2^m} \sum_{\pi \in \Pi_{2m}} \int_{S \in (\cX \times \cY)^{2m}} \mathds{1}_{\cK_{\epsilon,2m}} (\pi(S)) \,d \cD^{2m}(S) \\
     &= \int_{S \in (\cX \times \cY)^{2m}}  \frac{\sum_{\pi \in \Pi_{2m}} \mathds{1}_{\cK_{\epsilon,2m}} (\pi(S)) }{2^m} \,d \cD^{2m}(S)
\end{align*}
For a fixed $S \subseteq (\cX \times \cY)^{2m}$, if we denote $(\pi_0, \pi_1) \to T_S((\pi_0, \pi_1))$ ($\pi \to T_S(\pi)$ ) the random variable, such that $T_S(\pi) = \mathds{1}_{\cK_{\epsilon,2m}} (\pi(S))$, the term inside the integral can be seen as the expectation over $\Pi_{2m}$ with the uniform distribution of $T_S(\pi)$. 

We have,

\begin{align*}
     \underset{S \sim \cD^{2m}}{\prob} (\cK_{\epsilon,2m}) = \int_{S \in (\cX \times \cY)^{2m}}  \frac{\sum_{\pi \in \Pi_{2m}} T_S(\pi) }{2^m} \,d \cD^{2m}(S)
\end{align*}

\begin{align*}
    \underset{\pi \sim \cU(\Pi_{2m})}{\E} \Big (T_S(\pi) \Big) &= \underset{\substack{\pi_1 \sim \cU(\Pi_{2m_1}) \\ \pi_1 \sim \cU(\Pi_{2m_1})}}{\prob} \Big[\exists c \in \cC: | \frac{1}{m_0}\sum_{i=1}^{m_0} \mathds{1}_{c(x_{\pi_0(i)}) = 1} - \frac{1}{m_1}\sum_{i=1}^{m_1} \mathds{1}_{c(x_{\pi_1(i)}) = 1}\\
    &- \frac{1}{m_0} \sum_{i=1}^{m_0} \mathds{1}_{c(x_{\pi_0(m_0+i)}) = 1}  + \frac{1}{m_1} \sum_{i=1}^{m_1} \mathds{1}_{c(x_{\pi_1(m_1+i)}) = 1} | \geq \frac{ \epsilon}{2}  \Big] \\
    &= \underset{\substack{\pi_1 \sim \cU(\Pi_{2m_1}) \\ \pi_1 \sim \cU(\Pi_{2m_1})}}{\prob} \Big[\exists c \in \cC: | \frac{1}{m_0} \Big(\sum_{i=1}^{m_0} \mathds{1}_{c(x_{\pi_0(i)}) = 1} - \sum_{i=1}^{m_0} \mathds{1}_{c(x_{\pi_0(m_0+i)}) = 1} \Big)\\
    &- \frac{1}{m_1} \Big(\sum_{i=1}^{m_1} \mathds{1}_c(x_{\pi_1(i)}) = 1  - \sum_{i=1}^{m_1} \mathds{1}_{c(x_{\pi_1(m_1+i)}) = 1} \Big) | \geq \frac{ \epsilon}{2}\Big)  \Big]
\end{align*}

For each fixed $c \in \cC$, 

\begin{align*}
    \underset{\pi_0 \sim \cU(\Pi_{2m_0})}{\E} \Big[  \Big(\sum_{i=1}^{m_0} \mathds{1}_c(x_{\pi_0(i)}) =1 - \sum_{i=1}^{m_0} \mathds{1}_c(x_{\pi_0(m_0+i)}) &=1 \Big) \Big] = \frac{1}{2} \Big(\sum_{i=1}^{m_0} \mathds{1}_c(x_i) =1 - \sum_{i=1}^{m_0} \mathds{1}_{\mathds{1}_c(x_{m_0 +i}) =1} \Big) + \\
    & \quad \frac{1}{2} \Big(\sum_{i=1}^{m_0} \mathds{1}_c(x_{m_0+i}) =1 - \sum_{i=1}^{m_0} \mathds{1}_c(x_i) =1 \Big) \\
    & = 0
\end{align*}

By symmetry of the permutations between the first and second protected groups, we have:

\begin{align*}
     \underset{\pi_1 \sim \cU(\Pi_{2m_1})}{\E} \Big[  \Big(\sum_{i=1}^{m_1} \mathds{1}_c(x_{\pi_1(i)}) =1 - \sum_{i=1}^{m_1} \mathds{1}_c(x_{\pi_1(m_1+i)}) = 1 \Big) \Big] = 0
\end{align*}

By applying Discrepancy Hoeffding inequality again (Lemma~\ref{lemma:discchernoff}), for every $S \subseteq (\cX \times \cY)^{2m}$ and $c \in \cC$, 

\begin{equation*}
   \underset{\pi \sim \cU(\Pi_{2m})}{\E} \Big (T_S(\pi) \Big) \leq  \exp \frac{- \min(m_0,m_1)}{16}\epsilon^2
\end{equation*}

By applying the union bound,

\begin{align*}
     \underset{S \sim \cD^{2m}}{\prob} (\cK_{\epsilon,2m}) &\leq |\Delta_{\cC}^{\SP}(S)|  \exp \frac{- \min(m_0,m_1)}{16}\epsilon^2 \\
     & \leq \Pi^{\SP}_{\mathfrak{p}}(2 m_0, 2m_1) \exp \frac{- \min(m_0,m_1)}{16}\epsilon^2 
\end{align*}

The final step by applying Sauer lemma ( Claim~\ref{claimm}) twice on each protected group, which is true since an SP dichotomy exist when each of the points conditionned on each of the protected group define a dichotomy (can be classified following a concept $c$ in $\cC$).

\end{proof}

\subsection{Lower Bounds on Sample Complexity: Hardness of Weakly Auditing}

Let $(\epsilon, \delta) \in (0,1)^2$ and $\cC$ denote a concept class of VC dimension $d+1$.
Let $\cZ = \{x_0, x_1 \cdots, x_d\}$ be a subset of $\cX$ that shatters $\cC$. For any finite subset $S$ of $\cX$, we denote $S_0$ (resp. $S_1$) the subset of $S$ whose elements belong to the first (resp. second) protected group.

For any $c,c' \in \cC$, let $c \Delta_0^1 c' = \{(x,x') \in \cX_0 \times \cX_1: \mathds{1}_{x \in c} - \mathds{1}_{x' \in c} \neq \mathds{1}_{x \in c'} - \mathds{1}_{x' \in c'}\}$. Intuitively, this defines the set of pairs from the two protected groups, where $c$ and $c'$ behave differently \footnote{$c$ and $c'$ do not behave the same with respect to the pair $(x,x')$.}. For any concept $c$, we say that a concept $c'$ is SP- consistent with $(S,c)$ if for every pair $(x,x') \in S_0 \times S_1$, $\mathds{1}_{x \in c} - \mathds{1}_{x' \in c} = \mathds{1}_{x \in c'} - \mathds{1}_{x' \in c'}$

By Lemma~\ref{lemma:welldef}, there exists two points for $\cZ$ such that each one belongs to a different protected group. Without loss of generality, we assume that $x_0 \in \cX_0$ and $x_1 \in \cX_1$. Let $\cZ_0 = \cZ \cap \cX_0 $ and $\cZ_1 = \cZ \cap \cX_1$, and $d_0 = |\cZ_0|$, $d_1 = |\cZ_1|$, where $\cZ_0$ (resp. $\cZ_1$) is indexed by $\cI_0$ (resp. $\cI_1$).

Let $\cD_{\cX}$ denotes the marginal distribution on $\cX$ supported on $\cZ$, and defined as

\begin{center}
$\begin{cases}
    \underset{\cD_{\cX}}{\prob} \{x_0\} =  \frac{1 - 8 \epsilon}{2}  \\
    \underset{\cD_{\cX}}{\prob} \{x_1\} = 0 \\
    \forall i \in \{2,3 \cdots, d\}: \underset{\cD_{\cX}}{\prob} \{x_i\} = \frac{8 \epsilon}{d-1}
\end{cases}$
\end{center}

Since $\cD_{\cX}$ is supported on $\cZ$, we assume without loss of generality that $\cX = \cZ$ and $\cC \subseteq2^{\cZ}$.

Let $\cC_{0,1} = \Big\{ \{x_0,x_1\} \cup T, T \subseteq \{x_2,x_3, \cdots, x_d\}   \Big\}$

And $\Tilde{S}_m = \{S \subseteq \cX: |S_0| = m_0, |S_1| = m_1, m_0 \leq \frac{d_0}{2}, m_1 \leq \frac{d_1}{2}\}$

We further assume that the auditor $\A$ is (possibly) randomized, and takes as input $\omega$ denoting a sequence of boolean random variables sampled independently (i.e. $\A = \A(S;\omega)$). And without loss of generality, we assume that $(x_0 , x_1) \in \A(S;\omega)$ whenever concepts are selected from $\cC_{0,1}$.

For any fixed $c \in \cC$ and $S \in \Tilde{S}_m$, let $p = \underset{x \sim \cD_{\cX}}{\prob} \Big[ \A(S; \omega)  \Big]$ and $p' = \underset{x \sim \cD_{\cX}}{\prob} \Big[ c  \Big]$

\begin{align*}
    |\mu(\A(S;\omega)) - \mu(c)| &= | \underset{\substack{k \in \cI_0 \\ k \neq 0}}{\sum} \underset{x_k \in \A(S; \omega)}{\sum} p_k - \underset{\substack{k \in \cI_1 \\ k \neq 1}}{\sum} \underset{x_k \in \A(S; \omega)}{\sum} p_k - \underset{\substack{k \in \cI_0 \\ k \neq 0}}{\sum} \underset{x_k \in c}{\sum} p'_k + \underset{\substack{k \in \cI_1 \\ k \neq 1}}{\sum} \underset{x_k \in c}{\sum} p'_k| \\
    &= \frac{8 \epsilon}{d-1}| \underset{\substack{k \in \cI_0 \\ k \neq 0}}{\sum} \underset{x_k \in \A(S; \omega)}{\sum} 1 - \underset{\substack{k \in \cI_1 \\ k \neq 1}}{\sum} \underset{x_k \in \A(S; \omega)}{\sum} 1 - \underset{\substack{k \in \cI_0 \\ k \neq 0}}{\sum} \underset{x_k \in c}{\sum} 1 + \underset{\substack{k \in \cI_1 \\ k \neq 1}}{\sum} \underset{x_k \in c}{\sum} 1| \\
    & = \frac{8 \epsilon}{d-1} \Big| (x,x') \in \cZ_0 \times \cZ_1: (x,x') \in  \A(S; \omega) \Delta_0^1 c   \Big|
\end{align*}

On the other hand, suppose we have a uniform distribution over the concept class $\cC_{0,1}$. For each $c \in \cC_{0,1}$, there are exactly $2^{d_0 -m_0}$ (resp. $2^{d_1 -m_1}$) concepts that are consistent with $(S_0,c)$ (resp. $(S_1,c)$). For any of the couples $(x,x') \in \cZ_0 \times \cZ_1$ that are not in $S$, $\frac{1}{2}$ of the SP-consistent concepts will contain this couple and $\frac{1}{2}$ will not. We deduce that $\A(S; \omega)$ behaves the same as $c$ with respect to this couple for exactly $\frac{1}{2}$ of these $2^{d_0  - m_0 + d_1 - m_1}$ SP-consistent concepts. Therefore:

\begin{align*}
    \underset{c \sim \mathbb{U}_{\cC_{0,1}}}{\E} \Big[ \Big| (x,x') \in \cZ_0 \times \cZ_1: (x,x') \in  \A(S; \omega) \Delta_0^1 c   \Big| \Big] &\geq \frac{d_0 - m_0 + d_1 - m_1}{2} \\
    & = \frac{d - m}{2} \\
    & \geq \frac{d}{4}
\end{align*}

Where the second step follows from the fact that $S \in \Tilde{S}_m$. Since $d = d_0 +d_1$, we have shown that:

\begin{align}\label{ineqLB}
     \underset{c \sim \mathbb{U}_{\cC_{0,1}}}{\E} \Big[|\mu(\A(S;\omega)) - \mu(c)| \Big] \geq \frac{2 \epsilon d}{d-1}  \geq 2 \epsilon
\end{align}
This is true for every value of $S$ and $\omega$,

\begin{align*}
    \underset{c \sim \mathbb{U}_{\cC_{0,1}}, S, \omega}{\E} \Big[|\mu(\A(S;\omega)) - \mu(c)| \Big] \geq \frac{2 \epsilon d}{d-1}  \geq 2 \epsilon
\end{align*}

Therefore there exists $c \in \cC_{0,1}$, such that:

\begin{align*}
    \underset{ S, \omega}{\E} \Big[|\mu(\A(S;\omega)) - \mu(c)| \Big] \geq \frac{2 \epsilon d}{d-1}  \geq 2 \epsilon
\end{align*}

On the other hand, we have for all $(x,x') \in \cX_0 \times \cX_1$:

\begin{align*}
    \Big(\mathds{1}_{x \in \A(S; \omega)} = \mathds{1}_{x \in c} \land \mathds{1}_{x' \in \A(S; \omega)} = \mathds{1}_{x' \in c} \Big) \lor \Big(\mathds{1}_{x \in \A(S; \omega)} \neq \mathds{1}_{x \in c} \land \mathds{1}_{x' \in \A(S; \omega)} \neq \mathds{1}_{x' \in c} \Big)  \iff (x,x') \notin \A(S; \omega) \Delta_0^1 c
 \end{align*}

 We deduce,

\begin{align*}
    1 -  \underset{\substack{x \sim \cD_{\cX} \\ x' \sim \cD_{\cX}}}{\prob} \Big[ \A(S; \omega) \Delta_0^1 c |x\in \cX_0, x' \in \cX_1 \Big] = &\underset{\substack{x \sim \cD_{\cX} \\ x' \sim \cD_{\cX}}}{\prob}   \Big[\mathds{1}_{x \in \A(S; \omega)} = \mathds{1}_{x \in c} \land \mathds{1}_{x' \in \A(S; \omega)} = \mathds{1}_{x' \in c} \Big) \\
    & \quad \lor \Big(\mathds{1}_{x \notin \A(S; \omega)} = \mathds{1}_{x \in c} \land \mathds{1}_{x' \notin \A(S; \omega)} = \mathds{1}_{x' \in c} |x\in \cX_0, x' \in \cX_1 \Big] \\
    & \leq \underset{\substack{x \sim \cD_{\cX} \\ x' \sim \cD_{\cX}}}{\prob}   \Big[\mathds{1}_{x \in \A(S; \omega)} = \mathds{1}_{x \in c} \land \mathds{1}_{x' \in \A(S; \omega)} = \mathds{1}_{x' \in c} |x\in \cX_0, x' \in \cX_1\Big] \\
    & +  \underset{\substack{x \sim \cD_{\cX} \\ x' \sim \cD_{\cX}}}{\prob} \Big[\mathds{1}_{x \notin \A(S; \omega)} = \mathds{1}_{x \in c} \land \mathds{1}_{x' \notin \A(S; \omega)} = \mathds{1}_{x' \in c}|x\in \cX_0, x' \in \cX_1 \Big] \\
     & \leq 1 -  \underset{\substack{x \sim \cD_{\cX} \\ x' \sim \cD_{\cX}}}{\prob}   \Big[\mathds{1}_{x \notin \A(S; \omega)} = \mathds{1}_{x \in c} \lor \mathds{1}_{x' \notin \A(S; \omega)} = \mathds{1}_{x' \in c} |x\in \cX_0, x' \in \cX_1\Big] \\
    & +  \underset{\substack{x \sim \cD_{\cX} \\ x' \sim \cD_{\cX}}}{\prob} \Big[\mathds{1}_{x \notin \A(S; \omega)} = \mathds{1}_{x \in c} \land \mathds{1}_{x' \notin \A(S; \omega)} = \mathds{1}_{x' \in c} |x\in \cX_0, x' \in \cX_1\Big]
\end{align*}

\begin{align*}
    \underset{\substack{x \sim \cD_{\cX} \\ x' \sim \cD_{\cX}}}{\prob} \Big[ \A(S; \omega) \Delta_0^1 c |x\in \cX_0, x' \in \cX_1 \Big] \geq  &\underset{\substack{x \sim \cD_{\cX} \\ x' \sim \cD_{\cX}}}{\prob}   \Big[ \mathds{1}_{x \notin \A(S; \omega)} = \mathds{1}_{x \in c} \lor \mathds{1}_{x' \notin \A(S; \omega)} = \mathds{1}_{x' \in c} |x\in \cX_0, x' \in \cX_1\Big] \\
    & -  \underset{\substack{x \sim \cD_{\cX} \\ x' \sim \cD_{\cX}}}{\prob} \Big[\mathds{1}_{x \notin \A(S; \omega)} = \mathds{1}_{x \in c} \land \mathds{1}_{x' \notin \A(S; \omega)} = \mathds{1}_{x' \in c} |x\in \cX_0, x' \in \cX_1 \Big]\\
    & \geq \underset{x \sim \cD_{\cX}}{\prob}   \Big[\mathds{1}_{x \notin \A(S; \omega)} = \mathds{1}_{x \in c}|x\in \cX_0\Big] +  \underset{x' \sim \cD_{\cX}}{\prob} \Big[\mathds{1}_{x' \notin \A(S; \omega)} = \mathds{1}_{x' \in c} | x' \in \cX_1 \Big]
\end{align*}

\begin{lemma}
    \begin{align*}
        \underset{x \sim \cD_{\cX} }{\prob}   \Big[\mathds{1}_{x \notin \A(S; \omega)} = \mathds{1}_{x \in c}|x\in \cX_0 \Big] +\underset{x' \sim \cD_{\cX}}{\prob} \Big[\mathds{1}_{x' \notin \A(S; \omega)} = \mathds{1}_{x' \in c} |x' \in \cX_1 \Big] \geq |\mu(\A(S;\omega)) - \mu(c)|
    \end{align*}
\end{lemma}

\begin{proof}
    By triangle inequality, 

    \begin{align*}
        |\mu(\A(S;\omega)) - \mu(c)| &\leq |\mu^0(\A(S;\omega)) - \mu^0(c)| +|\mu^1(\A(S;\omega)) - \mu^1(c)| \\
        &\leq |\mu^0(\A(S;\omega)) - \mu^0(c)| +|\mu^1(\A(S;\omega)) - \mu^1(c)| \\
        & \leq   \underset{x \sim \cD_{\cX} }{\prob}   \Big[\mathds{1}_{x \notin \A(S; \omega)} = \mathds{1}_{x \in c}|x\in \cX_0 \Big] +\underset{x' \sim \cD_{\cX}}{\prob} \Big[\mathds{1}_{x' \notin \A(S; \omega)} = \mathds{1}_{x' \in c} |x' \in \cX_1 \Big]
    \end{align*}
\end{proof}

We have proved the following:

\begin{align}\label{eqLB0}
     |\mu(\A(S;\omega)) - \mu(c)| \leq  \underset{\substack{x \sim \cD_{\cX} \\ x' \sim \cD_{\cX}}}{\prob} \Big[ \A(S; \omega) \Delta_0^1 c |x\in \cX_0, x' \in \cX_1 \Big] 
\end{align}

On the other hand, since $\A(S; \omega)$ is always correct on $(x_0,x_1)$ with respect to $c$, then for all $(x,x') \in \cX_0 \times \cX_1$:

\begin{align*}
    (x,x') \in \A(S;\omega) \Delta_0^1 c \implies (x,x') \neq (x_0,x_1)
\end{align*}

We deduce:

\begin{align*}
    \underset{\substack{X \sim \cD_{\cX} \\ X' \sim \cD_{\cX}}}{\prob} \Big[\A(S;\omega) \Delta_0^1 c | X \in \cX_0, X' \in \cX_1 \Big] &\leq \underset{\substack{X \sim \cD_{\cX} \\ X' \sim \cD_{\cX}}}{\prob} \Big[X \neq x_0 \land X' \neq x_1| X \in \cX_0, X' \in \cX_1   \Big] \\
    & \leq \underset{X \sim \cD_{\cX}}{\prob} \Big[X \neq x_0]  \\
    &=8 \epsilon
\end{align*}

Implementing this in the inequality~\ref{eqLB0} gives:

\begin{align}\label{ineqLB'}
    |\mu(\A(S;\omega)) - \mu(c)| \leq 8 \epsilon
\end{align}

From the inequalities in \ref{ineqLB} and \ref{ineqLB'}, we get:

\begin{align*}
     2 \epsilon \leq \underset{S, \omega}{\E} \Big[|\mu(\A(S;\omega)) - \mu(c)| \Big] \leq 8 \epsilon \underset{S, \omega}{\prob} \Big[|\mu(\A(S;\omega)) - \mu(c)| > \epsilon \Big] + \epsilon \Big(   \underset{S, \omega}{\prob} \Big[|\mu(\A(S;\omega)) - \mu(c)| > \epsilon \Big] - 1\Big)
\end{align*}

Therefore, 

\begin{align*}
     \underset{S, \omega}{\prob} \Big[|\mu(\A(S;\omega)) - \mu(c)| > \epsilon | S \in \Tilde{S}_m \Big] > \frac{1}{7}
\end{align*}

On the other hand, 

\begin{align*}
      \underset{S, \omega}{\prob} \Big[|\mu(\A(S;\omega)) - \mu(c)| > \epsilon \Big] &= \underset{S, \omega}{\prob} \Big[|\mu(\A(S;\omega)) - \mu(c)| > \epsilon | S \in \Tilde{S}_m \Big] \underset{S}{\prob} \Big[ S \in \Tilde{S}_m \Big] \\
      & \geq \frac{1}{7}  \underset{S}{\prob}  \Big[ S \in \Tilde{S}_m \Big] 
\end{align*}

Let $T_0$ (resp. $T_1$) denote the number of realizations of  $\{x_2, x_3 \dots, x_d\}$ in a sample $S_0$ (resp. $S_1$) 

\begin{align*}
    \underset{\cD^m}{\prob} \Big\{\Tilde{S}_m \Big\} &\geq  \underset{\cD^m}{\prob} \Big\{T_0 \leq \frac{d_0}{2} \land T_1 \leq \frac{d_1}{2} \Big\} \\
    & = \underset{\cD^m}{\prob} \Big\{ \min(T_0 ,T_1) \leq \min( \frac{d_0}{2},  \frac{d_1}{2} ) \Big\}
\end{align*}

The last term is bounded by the probability that a binomial( $\min(m_0,m_1), 8 \epsilon)$ is less than $\min( \frac{d_0}{2},  \frac{d_1}{2} )$.

By applying Bernstein inequality \ref{lemma:bernstein},

\begin{align*}
     \underset{\cD^m}{\prob} \Big\{\Tilde{S}_m \Big\} &\geq 1 - \exp{- \frac{\min( \frac{d_0}{2},  \frac{d_1}{2} ) - 8\min(m_0,m_1) \epsilon^2 }{8 \min(m_0,m_1) \epsilon(1-8\epsilon)}}\\
     & \geq 1 - \exp{- \frac{\min( d_0,d_1 ) - 16\min(m_0,m_1) \epsilon^2 }{16 \min(m_0,m_1) \epsilon(1-8\epsilon)}}
\end{align*}

For $\min(m_0, m_1) = \frac{\min(d_0,d_1)}{32 \epsilon}$, we can simplify the right term and for $\epsilon< \frac{1}{8}, \delta < \frac{1}{20}$ the sample complexity $\min(m_0, m_1) > \frac{\min(d_0,d_1)}{32 \epsilon}$.

\section{Proofs for Strongly Auditable Classes}
\subsection{All Finite Classes are Strongly Auditable (Theorem~\ref{theorem:stronglyfinite})} \label{app:stronglyfinite}

We start by restating Theorem~\ref{theorem:stronglyfinite}:
\stronglyfinite*

\begin{proof}
    Let $m^{\text{corr}}$ (resp. $m^{\text{comp}}$) denote the sample complexity for correctness (resp. correctness).

     Let $S$ denote a sample of labeled instances by the true hypothesis (realizability assumption), $S_{|x}$ denotes the projection of $S$ on the input space $\cX$. 
    
    \paragraph{Step 1: Bounding sample complexity for correctness $m^{\text{corr}}$.}

    Let $\Tilde{\cH}$ denote the set $\{h \in \cH, \cL_{\cD}(h) > \epsilon  \}$,
    
    In this step we decouple the subclass $A[S]$ from $S$ using the set $\Tilde{\cH}$ by the showing the following:

    \begin{claim}
        $\{S_{|x}, \exists h \in A[S], \cL_{\cD}(h)> \epsilon\} \subseteq \{S_{|x}, \exists h \in \Tilde{\cH}, \cL_{S}(h) = 0\}$
    \end{claim}
    \begin{proof}
        By the realizability assumption, $A[S] = \{h \in \cH, \cL_S(h) = 0 \}$, which concludes the proof of the claim.
    \end{proof}

        We deduce,

        \begin{align*}
            \underset{S \sim \cD^{m^{\text{corr}}}}{\prob} \Biggl[\exists h \in A[S],  \cL_{\cD}(h) > \epsilon   \Biggl]   &\leq \underset{S \sim \cD^{m^{\text{corr}}}}{\prob} \Biggl[\exists h \in \Tilde{\cH},  \cL_{S}(h) =0   \Biggl] \\
            & \leq \sum_{h \in \Tilde{\cH}} \underset{S \sim \cD^{m^{\text{corr}}}}{\prob} \Biggl[  \cL_{S}(h) =0   \Biggl] \\
            & \leq \sum_{h \in \Tilde{\cH}} \underset{S \sim \cD^{m^{\text{corr}}}}{\prob} \Biggl[ \forall x \in S:  h(x) = y  \Biggl] \\
            & \leq \sum_{h \in \Tilde{\cH}} \prod_{i=1}^{m^{\text{corr}}} \underset{x_i \sim \cD}{\prob} \Biggl[  h(x_i) = y_i  \Biggl] \\
            & \leq \sum_{h \in \Tilde{\cH}} \prod_{i=1}^{m^{\text{corr}}} (1- \epsilon) \\
            &\leq |\cH| e^{- \epsilon m^{\text{corr}}}
        \end{align*}
    
The first inequality comes from the result of the claim, the second inequality comes from the union-bound, the fourth inequality comes from the i.i.d assumption, and the fifth inequality comes from the definition of $\Tilde{\cH}$.

This result shows that a sample size of $\cO\Bigl( \frac{1}{\epsilon} \log \frac{|\cH|}{\delta}\Bigl)$ is sufficient for correctness.
    
\paragraph{Step 2: Bounding sample complexity for completeness $m^{\text{comp}}$.} 

In the following proof, we use the same decoupling argument with an extra cost over $\epsilon$ proportion of samples that leads to a small true error but a "large" (nonzero) empirical error. 

Let $\Bar{\Tilde{\cH}}$ denote the complementary of the set $\Tilde{\cH}$ defined in the previous step. That is $\Bar{\Tilde{\cH}} = \{h \in \cH, \cL_{\cD}(h) \leq \epsilon   \}$, and set $\cT$ to be the set $\{S_{|x}, \exists h \in \Bar{\Tilde{\cH}} \setminus A[S],  \forall x \in S_{|x}: h(x) \neq y   \}$.
\begin{claim}
The following is true:

    $\{S_{|x}, \exists h \in \cH \setminus A[S], \cL_{\cD}(h) \leq \epsilon \} \subseteq \{S_{|x}, \exists h \in  \Bar{\Tilde{\cH}},  \cL_S(h) = 0\} \cup \cT$.
\end{claim}

\begin{proof}
    Let $S_{|x}$ and  $h \in \cH \setminus A[S]$ such that, $\cL_{\cD}(h) \leq \epsilon$

Since $h \notin A[S]$, $\cL_S(h) \neq 0$, but since $\cL_{\cD}(h) \leq \epsilon$, let $\cT$ denote the fraction of samples from $S_{|x}$ that induce the error, that is for all $x \in \cT, h(x) \neq y$.

We have partitioned $S_{|x}$ into a subset $S_1$ that agrees with the true function on the labels, that is $\cL_{S_1} = 0$, and another subset that verifies: $\underset{x \sim \cD}{\prob}(\cT) \leq \epsilon$, which concludes the proof of the claim.
\end{proof}

We deduce the following:

 \begin{align*}
            \underset{S \sim \cD^{m^{\text{comp}}}}{\prob} \Biggl[\exists h \in \cH \setminus A[S],  \cL_{\cD}(h) \leq \epsilon   \Biggl]   &\leq \underset{S \sim \cD^{m^{\text{comp}}}}{\prob} \Biggl[\exists h \in \Bar{\Tilde{\cH}},  \cL_{S}(h) =0   \Biggl] + \underset{S \sim \cD^{m^{\text{comp}}}}{\prob} \Bigl[\forall x \in S, x \in \cT \Bigl] \\
            & \leq \sum_{h \in \cH} \underset{S \sim \cD^{m^{\text{corr}}}}{\prob} \Biggl[  \cL_{S}(h) =0   \Biggl] + \sum_{h \in \cH} \prod_{i=1}^{m^{\text{corr}}} \underset{x_i \sim \cD}{\prob} \Bigl[ \cT \Bigl]\\
            & \leq \sum_{h \in \cH} \prod_{i=1}^{m^{\text{corr}}} (1- \epsilon) +  \sum_{h \in \cH} \prod_{i=1}^{m^{\text{corr}}} \epsilon\\
            &\leq |\cH| e^{- \epsilon m^{\text{corr}}} + |\cH| \epsilon^{m^{\text{corr}}}
        \end{align*}

This result shows that a sample size of $\cO \Bigl(\max \Big\{\frac{1}{\epsilon} \log \frac{|\cH|}{\delta}, \frac{1}{\log\frac{1}{\epsilon}} \log \frac{|\cH|}{\delta}\Big\}\Bigl)$ is sufficient for correctness.

\end{proof}

\subsection{Concentration Bounds on Prospect Ratio (Theorem~\ref{theorem:prospectratio})}\label{app:prospectratio}

We first begin by restating the theorem: 
\asymsize*

\begin{proof} 
Let $\epsilon, \upsilon, \tau \in (0,1)$.
    By the definition of the estimator of prospect ratio:
    $$\hat{\Tilde{r}}_{n,m_0,m_1}(\epsilon) \triangleq \frac{1}{n} \sum_{i=1}^n \mathds{1}_{|\frac{1}{m_0} \sum_{i=1}^{m_0} \mathds{1}_{f_k(x_i) = 1} - \frac{1}{m_1} \sum_{j=1}^{m_1} \mathds{1}_{f_k(x'_j)=1}| \leq \epsilon}$$
    Since we sample independently $f_k$'s from $\cF$, the random variables $\mathds{1}_{|\frac{1}{m_0} \sum_{i=1}^{m_0} \mathds{1}_{f_k(x_i) = 1} - \frac{1}{m_1} \sum_{j=1}^{m_1} \mathds{1}_{f_k(x'_j)=1}| \leq \epsilon}$ are also independent and take values in $[0,1]$. 

    By Hoeffding inequality:

    \begin{equation}\label{eq:aprime}
        \prob \{| \Tilde{r}_n(\epsilon)  - r(\epsilon)   |\geq \tau    \} \leq 2 \exp \{ - 2 n \tau^2 \}
    \end{equation}

    On the other hand, let $S$ be sample that contains $m_0$ points from first protected group and $m_1$ points from second protected group. By applying Lemma~\ref{lemma:discchernoff} on each function $f_k$ over equally size subsamples of size $\frac{m}{n}$, we have for all $k \in [n]$:

    $$\underset{S \sim \cD^{\frac{m}{n}}}{\prob}  \{| \hat{\mu}_S(f_k) - \mu_{\cD}(f_k)| \geq \upsilon  \}  \leq \exp \Big\{ \frac{-2 \upsilon^2 m_0 m_1}{n(m_0 + m_1)} \Big \}$$
    
    By the independence of the events:
    \begin{equation}\label{eq:a}
        \underset{S \sim \cD^{\frac{m}{n}}}{\prob}  \{\inf_{k \in [N]}| \hat{\mu}_S(f_k) - \mu_{\cD}(f_k)| \geq \upsilon  \}   = \prod_{k=1}^n \underset{S \sim \cD^{\frac{m}{n}}}{\prob}  \{| \hat{\mu}_S(f_k) - \mu_{\cD}(f_k)| \geq \upsilon  \}  
    \end{equation}

    We deduce:

    $$\underset{S \sim \cD^{\frac{m}{n}}}{\prob}  \{\sup_{k \in [N]}| \hat{\mu}_S(f_k) - \mu_{\cD}(f_k)| \leq \upsilon  \}   \geq \Big( 1 - \exp \Big\{ \frac{-2 \upsilon^2 m_0 m_1}{n(m_0 + m_1)} \Big \}\Big)^n $$

    Now let $k \in [N]$, such that $ |\hat{\mu}_S(f_k) - \mu_{\cD}(f_k)| \geq \upsilon$
    
    for all $\tau$ in $(0,1)$, 
    \begin{align*}
        f_k \in \hat{\cP}(\cF, \epsilon) &\implies |\hat{\mu} - \mu (f_k)| \leq \epsilon \leq \epsilon + \upsilon\\
        & \implies f_k \in \hat{\cP}(\cF, \epsilon + \upsilon)\\
    \end{align*}

    Since this is true for all $k$'s, we deduce: $\hat{\Tilde{r}} \leq \Tilde{r}(\epsilon + \upsilon)$

    Similarly if $f_k \in \hat{\cP}(\cF, \epsilon - \upsilon)$, we have $\Tilde{r}(\epsilon - \upsilon) \leq \hat{\Tilde{r}}(\epsilon)$
    In other words,  $ \Tilde{r}(\epsilon - \upsilon) \leq \hat{\Tilde{r}}(\epsilon) \leq \Tilde{r}(\epsilon + \upsilon)$

    Therefore, 
    \begin{equation}\label{eq:b}
        \underset{S \sim \cD^{\frac{m}{n}}}{\prob}  \{\sup_{k \in [N]}| \hat{\mu}_S(f_k) - \mu_{\cD}(f_k)| \leq \upsilon  \} \leq \prob \{ \Tilde{r}(\epsilon - \upsilon) \leq \hat{\Tilde{r}}(\epsilon) \leq \Tilde{r}(\epsilon + \upsilon) \}
    \end{equation}

    From Equation \ref{eq:a} and Equation\ref{eq:b}, we deduce: 

\begin{equation}\label{eq:c}
        \Big( 1 - \exp \Big\{ \frac{-2 \upsilon^2 m_0 m_1}{n(m_0 + m_1)} \Big \}\Big)^n  \leq \prob \{ \Tilde{r}(\epsilon - \upsilon) \leq \hat{\Tilde{r}}(\epsilon) \leq \Tilde{r}(\epsilon + \upsilon) \}
    \end{equation}

By the inequality in \ref{eq:aprime} and the inequality \ref{eq:c}, we deduce the desired result.

\end{proof}

\subsection{Infinite VC Classes are Not Auditable (Proposition \ref{prop:infinitevcnotsp})}\label{app:infinitevcnotsp}

We start be restating Proposition~\ref{prop:infinitevcnotsp}.

\infinitevcnotsp*

\begin{proof}
We prove the result by showing a lower bound on the SP dimension that depends on the VC dimension.

Let $\cF$ denotes a hypothesis class and $S$ denotes a sample that SP-shatters $\cF$, $S_0$ (resp. $S_1$) the subset of $S$ belonging to the first protected group (resp. second protected group).

Since $S$ SP-shatters $\cF$, by the result in Lemma~\ref{lemma:welldef}, we have 
\begin{align*}
   & \max(|S_0|, |S_1|) \leq |S| - 2 \\
    & 2^{\max(|S_0|, |S_1|)} \leq  2^{|S|} (1 -  \frac{3}{4})
\end{align*}
Hence, for all $S$ that SP-shatters $ \cF$:
\begin{align*}
    2^{|S|} - 2^{\max(|S_0|, |S_1|)} \geq \frac{3}{4} 2^{|S|}
\end{align*}

And since $$ 2^{|S|} - 2^{\max(|S_0|, |S_1|)}  \leq 2^{|S|} - 2^{|S_0|} - 2^{|S_1|} $$

We deduce 
\begin{align*}
    \SP(\cF) \geq \log_2 \frac{3}{4} + \VC(\cF)
\end{align*}
This implies that if the VC dimension is infinite then SP dimension is also infinite.
\end{proof}

\section{Extended Technical Details for Statistical Parity Audit}\label{app:extexp}
\subsection{Hardness of Identifying Prospect Class} 

When the prospect class is infinite, any algorithm attempting to exhaustively evaluate all models in the prospect class would never terminate. This challenge can be illustrated through linear classifiers in two dimensions, as shown in Figure~\ref{fig:hardprospect}. In this example, any line passing through the blue region belongs to the prospect class, resulting in infinitely many candidate models. 

\begin{figure}[t!]
    \centering
    \includegraphics[width= 0.4\columnwidth]{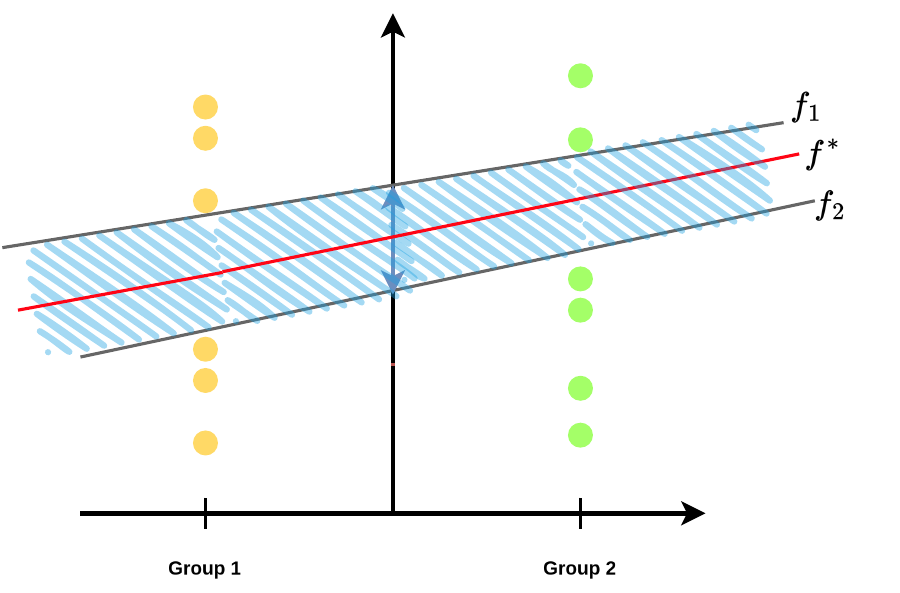}
    \caption{Illustration of prospect class for linear classifiers. Group 1 (yellow) and Group 2 (green) represent the two protected groups. Any classifier in the region delimited by $f_1$ and $f_2$ has the same statistical parity value.}\label{fig:hardprospect}
\end{figure}

\subsection{Relationship between Prospect Ratio and SP-dimension} 

The prospect ratio and SP dimension serve fundamentally different purposes in our analysis. While the SP dimension, like other complexity measures such as VC dimension, characterizes the capacity of a hypothesis class to generalize finite-sample properties to distributional ones, the prospect ratio serves a distinct role. Specifically, it quantifies the likelihood of finding models that satisfy post-audit requirements while maintaining equivalent properties under audit values. A key distinction is that the prospect ratio is inherently data dependent, whereas the SP dimension is determined uniquely by the hypothesis class structure.
\section{Extended Experimental Details and Datasets}\label{app:extendexp}
\begin{figure*}[ht]
\centering
\setlength{\tabcolsep}{4pt} 
\renewcommand{\arraystretch}{1.3} 
\begin{tabular}{cccc}
\includegraphics[width=0.22\textwidth]{Figures/prospect_ratio_E1.pdf} &
\includegraphics[width=0.22\textwidth]{Figures/prospect_ratio_E2.pdf} &
\includegraphics[width=0.22\textwidth]{Figures/prospect_ratio_E3.pdf} &
\includegraphics[width=0.22\textwidth]{Figures/prospect_ratio_E4.pdf} \\[0.3cm]

\includegraphics[width=0.22\textwidth]{Figures/estimation_error_E1.pdf} &
\includegraphics[width=0.22\textwidth]{Figures/estimation_error_E2.pdf} &
\includegraphics[width=0.22\textwidth]{Figures/estimation_error_E3.pdf} &
\includegraphics[width=0.22\textwidth]{Figures/estimation_error_E4.pdf} \\[0.3cm]

\includegraphics[width=0.22\textwidth]{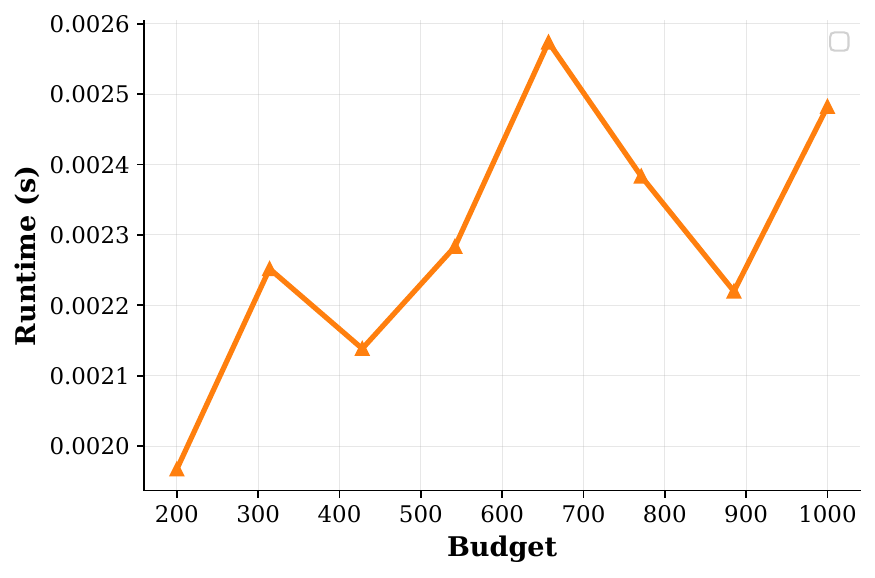} &
\includegraphics[width=0.22\textwidth]{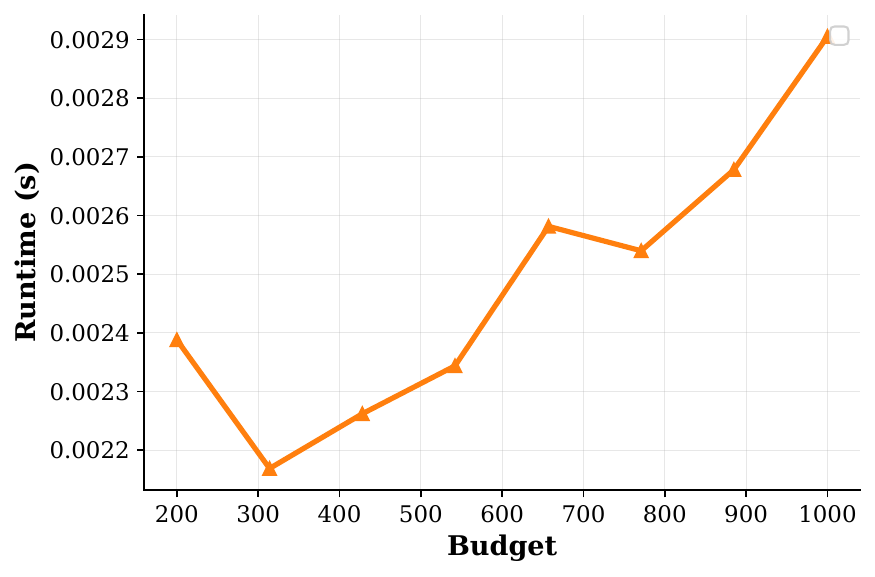} &
\includegraphics[width=0.22\textwidth]{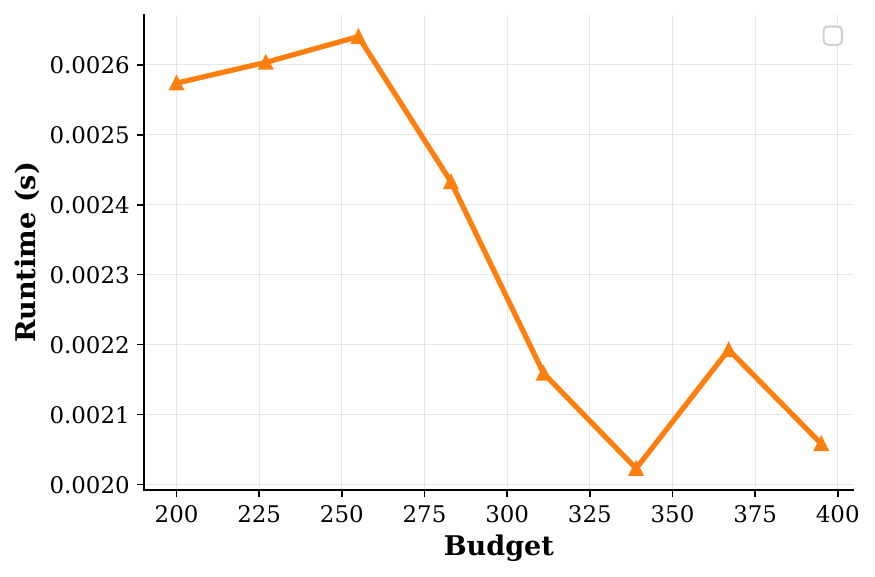} &
\includegraphics[width=0.22\textwidth]{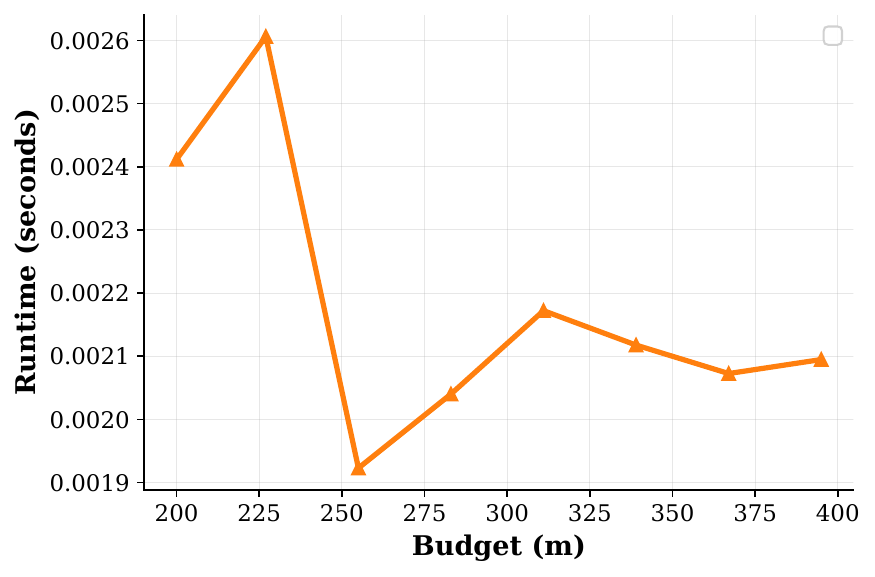} \\
\parbox[c]{0.22\textwidth}{\centering (a) COMPAS dataset\\[2pt] MLPs as strategic set} &
\parbox[c]{0.22\textwidth}{\centering (b) COMPAS dataset\\[2pt] RFs as strategic set} &
\parbox[c]{0.22\textwidth}{\centering (c) Student dataset\\[2pt] MLPs as strategic set} &
\parbox[c]{0.22\textwidth}{\centering (d) Student dataset\\[2pt] RFs as strategic set} \\[3pt]
\end{tabular}

\caption{Comparison of Runtimes, errors in statistical parity estimation, and prospect ratio across different datasets and strategic sets.}
\label{fig:bis}
\end{figure*}

All experiments were conducted on an 11th Gen Intel\textsuperscript{\textregistered} Core\texttrademark{} i7-1185G7 processor (3.00 GHz, 8 cores) with 32.0 GiB of RAM. Implementation details and instructions for reproducing our results are provided in the supplementary code repository: \url{https://anonymous.4open.science/r/Auditors-with-prospects-050F}.

Figure~\ref{fig:bis} extends Figure~\ref{fig:experiments} by reporting runtimes across datasets and strategic set transitions. In Experiment (a), the true prospect ratio (defined as the cardinality of the true prospect set divided by that of the sampled model from the strategic class) is 0.075. Despite this sparsity, Algorithm~\ref{algo} accurately identifies the prospect set using only 200  samples. Notably, estimation accuracy improves monotonically with sample size, while runtime remains stable, exhibiting only bounded, non-monotonic fluctuations. In Experiment (b), where the true prospect ratio increases to 0.125, the algorithm similarly recovers the prospect subset with high fidelity, achieving vanishing statistical error and low computational overhead. We further validate our approach on the Student Performance dataset~\citep{cortez2008using}, using gender as the protected attribute, and observe consistent behavior: the auditor reliably captures the full prospect class with minimal runtime and negligible estimation error. 

These findings align with the broader empirical analysis in Section~\ref{sec:experiments}. As shown in Figures~\ref{fig:experiments} and \ref{fig:bis}, the auditor effectively recovers the entire prospective class, as measured by the prospect ratio. Concurrently, it preserves statistical parity, demonstrating stability of fairness properties under audit. Crucially, this is achieved without sacrificing flexibility in model updates and predictive performance, illustrating a favorable trade-off between accuracy and completeness. 
 
\section{Useful Technical Results}
\begin{claim}\label{claiminterprob}
    For all $m_0, m_1 \in \R^{+}$, $$\min(m_0, m_1) \leq \frac{2 m_0 m_1}{m_0 + m_1}$$
\end{claim}

\begin{claim}\label{claim:separation}
    For $a \geq 1,b>0$, if the following holds:
    $$x \geq 4a \log 2a +2b$$
    We have:
    $$x \geq a \log x + b$$
\end{claim}
\begin{claim}[Bernstein inequality]\label{lemma:bernstein}
When $X_1, \cdots, X_n$ are independent random variables, with $\prob[X_i] \leq M$ almost surely for all $i \in [n]$, the following holds
    $$\mathbb{P}\left(\left|\sum_{i=1}^{n} (X_i - \mathbb{E}[X_i])\right| \geq t\right) \leq 2\exp\left(-\frac{t^2/2}{\sum_{i=1}^{n} \mathbb{E}[X_i^2] + Mt/3}\right)$$
\end{claim}

\begin{lemma}[Discrepancy Chernoff bounds]\label{lemma:discchernoff}
    If $Q_1, \cdots, Q_{m_0}, R_1, \cdots, R_{m_1}$ are independent random variables taking values in $[0,1]$, then:

    $$\prob \Big[Q_{m_0} - R_{m_1} - (\E Q_{m_0} - \E R_{m_1}   ) >\epsilon\Big]  \leq \exp{\frac{- 2 m_0 m_1 \epsilon^2}{m_0 + m_1}}$$
\end{lemma}

The following claim will serve to derive upper bounds for improper auditing.
\begin{claim}\label{claimm}
    For all $m,s \in \N$,
    \begin{align*}
         \sum_{i=0}^s \binom{m}{i} \leq \Big(\frac{en}{s}\Big)^s
    \end{align*}
\end{claim}

\end{document}